\documentclass{article}

\usepackage{hyperref}
\usepackage[accepted]{arxiv2018}

\usepackage{url}

\usepackage{multirow}
\usepackage[utf8]{inputenc} 
\usepackage[T1]{fontenc}    
\usepackage{booktabs}       
\usepackage{nicefrac}       
\usepackage{microtype}      
\usepackage{mathtools}

\usepackage{tabularx,ragged2e}
\newcolumntype{C}[1]{>{\Centering}m{#1}}

\usepackage{listings}
\usepackage{amsthm}
\usepackage{amssymb}
\usepackage{framed} 
\usepackage{amsmath}
\usepackage{amssymb}
\usepackage{relsize}
\usepackage{tikz}
\usepackage{xr}
\usepackage{tikz-cd}

\usepackage{times}
\usepackage{graphicx} 
\usepackage{subfigure} 
\usepackage{wrapfig}
\usepackage{algpseudocode}
\usepackage{algpascal}

\newtheorem{theorem}{Theorem}

\newtheorem{definition}[theorem]{Definition}

\numberwithin{theorem}{section}
\numberwithin{equation}{section}

\usetikzlibrary{backgrounds}
\usetikzlibrary{calc}
\usepackage{soul}
\usepackage{stackrel}

\usepackage{relsize}

\tikzset{fontscale/.style = {font=\relsize{#1}}
    }

\def\t2{\tfrac12}

\def\scriptf{{\mathcal F}}

\def\scripta{{\mathcal A}}

\def\scriptd{{\mathcal D}}

\def\scriptn{{\mathcal N}}

\arxivtitlerunning{On Characterizing the Capacity of Neural Networks using Algebraic Topology}
\title{On Characterizing the Capacity of Neural Networks using Algebraic Topology}

\begin{document}

\twocolumn[
\arxivtitle{On Characterizing the Capacity of Neural Networks using Algebraic Topology}

\arxivsetsymbol{equal}{*}

\begin{arxivauthorlist}
\arxivauthor{William H.~Guss}{cmu} \arxivauthor{Ruslan Salakhutdinov}{cmu} \\
\texttt{\{wguss, rsalakhu\}@cs.cmu.edu}
\end{arxivauthorlist}

\arxivcorrespondingauthor{William Guss}{wguss@cs.cmu.edu}

\arxivaffiliation{cmu}{Machine Learning Department, Carnegie Mellon University, Pittsburgh, Pennsylvania}

\arxivkeywords{Machine Learning, arxiv}

\vskip 0.3in
]

\printAffiliationsAndNotice{}

\begin{abstract}%
The learnability of different neural architectures can be characterized directly by computable measures of data complexity. In this paper, we reframe the problem of architecture selection as understanding how data determines the most expressive and generalizable architectures suited to that data, beyond inductive bias. After suggesting algebraic topology as a measure for data complexity, we show that the power of a network to express the topological complexity of a dataset in its decision region is a strictly limiting factor in its ability to generalize. We then provide the first empirical characterization of the topological capacity of neural networks. Our empirical analysis shows that at every level of dataset complexity, neural networks exhibit topological phase transitions. This observation allowed us to connect existing theory to empirically driven conjectures on the choice of architectures for fully-connected neural networks. 
\end{abstract}

\section{Introduction}

Deep learning has rapidly become one of the most pervasively applied techniques in machine learning. From computer vision \cite{krizhevsky2012imagenet} and reinforcement learning \cite{mnih2013playing} to natural language processing \cite{wu2016google} and speech recognition \cite{hinton2012deep}, the core principles of hierarchical representation and optimization central to deep learning have revolutionized the state of the art; see \citet{Goodfellow-et-al-2016}. In each domain, a major difficulty lies in selecting the architectures of models that most optimally take advantage of structure in the data. In computer vision, for example, a large body of work (\cite{DBLP:journals/corr/SimonyanZ14a}, \cite{DBLP:journals/corr/SzegedyLJSRAEVR14}, \cite{DBLP:journals/corr/HeZRS15}, etc.) focuses on improving the initial architectural choices of \citet{krizhevsky2012imagenet} by developing novel network topologies and optimization schemes specific to vision tasks.  Despite the success of this approach, there are still not general principles for choosing architectures in arbitrary settings, and in order for deep learning to scale efficiently to new problems and domains without expert architecture designers, the problem of architecture selection must be better understood.

Theoretically, substantial analysis has explored how various properties of neural networks, (eg. the depth, width, and connectivity) relate to their expressivity and generalization capability (\cite{raghu2016expressive}, \cite{daniely2016toward}, \cite{guss2016deep}). However, the foregoing theory can only be used to determine an architecture in practice if it is understood how expressive a model need be in order to solve a problem. On the other hand, neural architecture search (NAS) views  architecture selection as a compositional hyperparameter search (\cite{DBLP:journals/corr/SaxenaV16}, \cite{DBLP:journals/corr/FernandoBBZHRPW17}, \cite{45826}). As a result NAS ideally yields expressive and powerful architectures, but it is often difficult to interpret the resulting architectures beyond justifying their use from their empirical optimality.

We propose a third alternative to the foregoing:  data-first architecture selection. In practice, experts design architectures with some inductive bias about the data, and more generally, like any hyperparameter selection problem, the most expressive neural architectures for learning on a particular dataset are solely determined by the nature of the true data distribution. Therefore, architecture selection can be rephrased as follows: \emph{given a learning problem (some dataset), which architectures are suitably regularized and expressive enough to learn and generalize on that problem? } 

\begingroup

A natural approach to this question is to develop some objective measure of data complexity, and then characterize neural architectures by their ability to learn subject to that complexity. Then given some new dataset, the problem of architecture selection is distilled to computing the data complexity and choosing the appropriate architecture. 

For example, take the two datasets $\scriptd_1$ and $\scriptd_2$ given in Figure \ref{fig:increasing_complexity}(a,b) and Figure \ref{fig:increasing_complexity}(c,d) respectively. The first dataset, $\scriptd_1$, consists of positive examples sampled from two disks and negative examples from their compliment. On the right, dataset $\scriptd_2$ consists of positive points sampled from two disks and two rings with hollow centers. Under some geometric measure of complexity $\scriptd_2$ appears more 'complicated' than $\scriptd_1$ because it contains more holes and clusters. As one trains single layer neural networks of increasing hidden dimension on both datasets, \emph{the minimum number of hidden units required to achieve zero testing error is ordered according to this geometric complexity.} Visually in Figure \ref{fig:increasing_complexity}, regardless of initialization no single hidden layer neural network with $\leq $ 12 units, denoted $h_{\leq 12}$, can express the two holes and clusters in $\scriptd_2$. Whereas on the simpler $\scriptd_1$, both $h_{12}$ and $h_{26}$ can express the decision boundary perfectly. Returning to architecture selection, one wonders if this characterization can be extrapolated; that is, is it true that for datasets with 'similar' geometric complexity to $\scriptd_1$, any architecture with $\geq$ 12 hidden learns perfectly, and likewise for those datasets similar in complexity to $\scriptd_2$, architectures with $\leq 12$ hidden units can never learn to completion?

\subsection{Our Contribution}

In this paper, we formalize the above notion of geometric complexity in the language of algebraic topology. We show that questions of architecture selection can be answered by understanding the 'topological capacity' of different neural networks. In particular, a geometric complexity measure, called persistent homology, characterizes the capacity of neural architectures in direct relation to their ability to generalize on data. Using persistent homology, we develop a method which  gives the first empirical insight into the learnability of different architectures as data complexity increases. In addition, our method allows us to generate conjectures which tighten known theoretical bounds on the expressivity of neural networks. Finally, we show that topological characterizations of architectures areuseful in practice by presenting a new method, topological architecture selection, and applying it to several OpenML datasets.

\endgroup

\section{Background}

\subsection{General Topology}

In order to more formally describe notions of geometric complexity in datasets, we will turn to the language of topology. Broadly speaking, topology is a branch of mathematics that deals with characterizing shapes, spaces, and sets by their \emph{connectivity}. In the context of characterizing neural networks, we will work towards defining the topological complexity of a dataset in terms of how that dataset is 'connected', and then group neural networks by their capacity to produce decision regions of the same connectivity.

\begingroup

In topology, one understands the relationships between two different spaces of points by the \emph{continuous maps} between them.  Informally, we say that two topological spaces $A$ and $B$ are \emph{equivalent} ($A \cong B$) if there is a continuous function $f: A \to B$ that has an inverse $f^{-1}$ that is also continuous. When $f$ exists, we say that $A$ and $B$ are \emph{homeomorphic} and $f$ is their \emph{homeomorphism}; for a more detailed treatment of general topology see \citet{bredon2013topology}. In an informal way, $\scriptd_1 \not \cong \scriptd_2$ in Figure \ref{fig:increasing_complexity} since if there were a homeomorphism $f: \scriptd_1 \to \scriptd_2$ at least one of the clusters in $\scriptd_1$ would need to be split discontinuously in order to produce the four different regions in~$\scriptd_2$.

\begin{figure}[t]
	\begin{center}
		\includegraphics[width=0.47\textwidth]{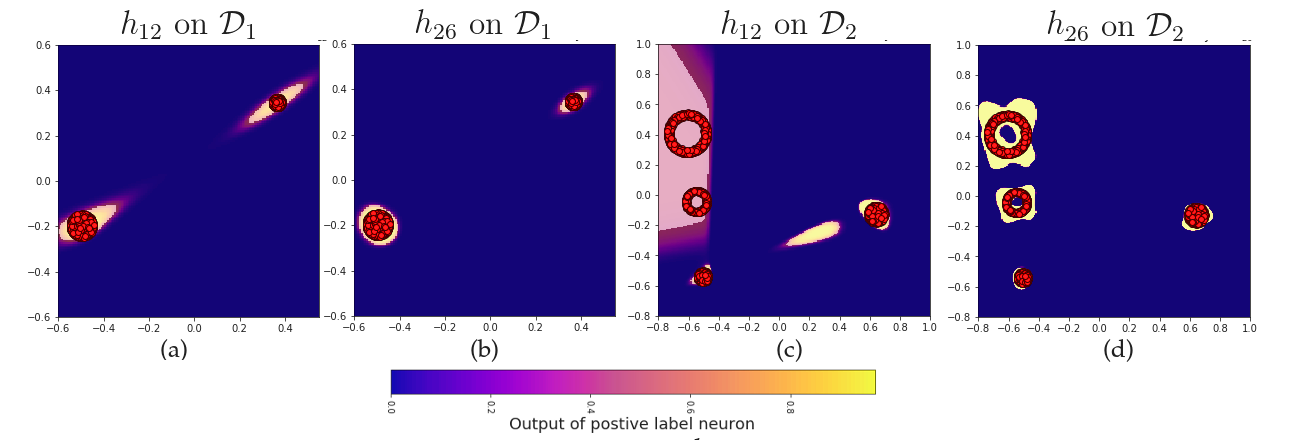}
	\end{center}
\vspace{-0.1in}
	\caption{The positive label outptus of single hidden layer neural networks, $h_{12}$ and $h_{26}$, of 2 inputs with 12 and 26 hidden units  respectively after training on datasets $\scriptd_1$ and $\scriptd_2$ with positive examples in red. Highlighted regions of the output constitute the positive decision region.}
	\label{fig:increasing_complexity}
\end{figure}

The power of topology lies in its capacity to differentiate sets (topological spaces) in a meaningful geometric way that discards certain irrelevant properties such as rotation, translation, curvature, etc. For the purposes of defining geometric complexity, non-topological properties\footnote{A \emph{topological property} or \emph{invariant} is one that is preserved by a homeomorphism. For example, the number of holes and regions which are disjoint from one another are topological properties, whereas curvature is not.\label{foot:top_prop}} like curvature would further fine-tune architecture selection--say if $\scriptd_2$ had the same regions but with squigly (differentially complex) boundaries, certain architectures might not converge--but as we will show, grouping neural networks by 'topological capacity' provides a powerful minimality condition. That is, we will show that if a certain architecture is incapable of expressing a decision region that is equivalent in topology to training data, then there is no hope of it ever generalizing to the true data.
\endgroup
\subsection{Algebraic Topology}

Algebraic topology provides the tools necessary to not only build the foregoing notion of topological equivalence into a measure of geometric complexity, but also to compute that measure on real data (\cite{betti1872nuovo}, \cite{dey1998computational}, \cite{bredon2013topology}). At its core, algebraic topology takes topological spaces (shapes and sets with certain properties) and assigns them algebraic objects such as \emph{groups}, \emph{chains}, and other more exotic constructs. In doing so, two spaces can be shown to be topologically equivalent (or distinct) if the algebraic objects to which they are assigned are isomorphic (or not). Thus algebraic topology will allow us to compare the complexity of decision boundaries and datasets by the objects to which they are assigned.

Although there are many flavors of algebraic topology, a powerful and computationally realizable tool is homology.
\begin{definition}[Informal, \cite{bredon2013topology}]
	If $X$ is a topological space, then $H_n(X) = \mathbb{Z}^{\beta_n}$ is called \textbf{the $n^{th}$ \emph{homology group} of $X$} if the power $\beta_n$ is the number of 'holes' of dimension $n$ in $X$. Note that $\beta_0$ is the number of separate connected components. We call $\beta_n(X)$ the $n$th Betti number of $X$. Finally, the homology\footnote{This definition of homology makes many assumptions on $X$ and the base field of computation, but for introductory purposes, this informality is edifying.} of $X$ is defined as $H(X) = \{H_n(X)\}_{n=0}^\infty.$
	\label{def:homology}
\end{definition}  

Immediately homology brings us closer to defining the complexity of $\scriptd_1$ and $\scriptd_2$. If we assume that $\scriptd_1$ is not actually a collection of $N$ datapoints, but really the union of $2$ solid balls, and likewise that $\scriptd_2$ is the union of $2$ solid balls and 2 rings, then we can compute the homology directly. In this case $H_0(\scriptd_1) = \mathbb{Z}^2$ since there are two connected components\footnote{Informally, a \emph{connected component} is a set which is not contained in another connected set except for itself.}; $H_1(\scriptd_1) = \{0\}$ since there are no circles (one-dimensional holes); and clearly, $H_n(\scriptd_1) = \{0\}$ for $n \geq 2 $. Performing the same computation in the second case, we get $H_0(\scriptd_2) = \mathbb{Z}^4$ and $H_1(\scriptd_2) = \mathbb{Z}^2$ as there are $4$ seperate clusters and $2$ rings/holes. With respect to any reasonable ordering on homology, $\scriptd_2$ is more complex than $\scriptd_1$. The measure yields non-trivial differentiation of spaces in higher dimension. For example, the homology of a hollow donut is $\{\mathbb{Z}^1,\mathbb{Z}^2,\mathbb{Z}^1,0, \dots\}$.

Surprisingly, the homology of a space contains a great deal of information about its topological complexity\footnotemark[\getrefnumber{foot:top_prop}]. 
The following theorem suggests the absolute power of homology to group topologically similar spaces, and therefore neural networks with topologically similar decision regions. 
\begin{theorem}[Informal]\label{thm:top_to_hom}
Let $X$ and $Y$ be topological spaces. If $X \cong Y$ then $H(X) = H(Y)$.\footnote{Equality of $H(X)$ and $H(Y)$ should be interpreted as isomorphism between each individual $H_i(X)$ and $H_i(Y)$.}
\end{theorem}

Intuitively, Theorem \ref{thm:top_to_hom} states that  number of  'holes' (and in the case of $H_0(X)$, connected components) are topologically invariant, and can be used to show that two shapes (or decision regions) are different.

\subsection{Computational Methods for Homological Complexity}

In order to compute the homology of both $\scriptd_1$ and $\scriptd_2$ we needed to assume that they were actually the geometric shapes from which they were sampled. Without such assumptions, \emph{for any dataset $\scriptd$} a $H(\scriptd) = \{\mathbb{Z}^N, 0, \dots\}$ where $N$ is the number of data points. This is because, at small enough scales each data point can be isolated as its own connected component; that is, as sets each pair of different positive points $d_1, d_2 \in \scriptd$ are disjoint. To properly utilize homological complexity in better understanding architecture selection, we need to be able to compute the homology of the data directly and still capture meaningful topological information.

\begingroup

\begin{figure}[t]
	\begin{center}
		\includegraphics[width=0.48\textwidth]{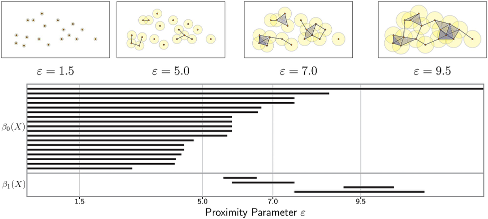}
\vspace{-0.1in}
		\caption{An illustration of computing persistent homology on a collection of points (\cite{topaz2015topological})} \label{fig:persistenthomology}
	\end{center}
\end{figure}

Persistent homology, introduced in \citet{zomorodian2005computing}, avoids the trivialization of computation of dataset homology by providing an algorithm to calculate the homology of a \emph{filtration} of a space. Specifically, a filtration is a topological space $X$ equipped with a sequence of subspaces $X_0 \subset X_1 \subset \dots \subset X$. In Figure \ref{fig:persistenthomology} one such particular filtration is given by growing balls of size $\epsilon$ centered at each point, and then letting $X_\epsilon$ be the resulting subspace in the filtration. Define $\beta_n(X)$ to be the $n$th Betti number of the homology $H(X_\epsilon)$ of $X_\epsilon$. Then for example at $\epsilon = 1.5$, $\beta_0(X_{\epsilon}) = 19$ and $\beta_1(X_{\epsilon}) = 0$ as every ball is disjoint. At $\epsilon = 5.0$ some connected components merge and $\beta_0(X_\epsilon) = 12$ and $\beta_1(X_\epsilon) = 0$. Finally  at $\epsilon = 7$, the union of the balls forms a hole towards the center of the dataset and $\beta_1(X_{\epsilon}) > 0$ with $\beta_0(X_\epsilon) =4.$

All together the change in homology and therefore Betti numbers for $X_{\epsilon}$ as $\epsilon$ changes can be summarized succinctly in the \emph{persistence barcode diagram} given in Figure \ref{fig:persistenthomology}. Each bar in the section $\beta_n(X)$ denotes a 'hole' of dimension $n$. The left endpoint of the bar is the point at which homology detects that particular component, and the right endpoint is when that component becomes indistinguishable in the filtration. When calculating the persistent homology of datasets we will frequently use these diagrams.

With the foregoing algorithms established, we are now equipped with the tools to study the capacity of neural networks in the language of algebraic topology.

\endgroup

\section{Homological Characterization of Neural Architectures}

In the forthcoming section, we will apply persistent homology to empirically characterize the power of certain neural architectures. To understand why homological complexity is a powerful measure for differentiating architectures, we present the following principle.

Suppose that $\scriptd$ is some dataset drawn from a joint distribution $F$ with continuous CDF on some topological space $X \times \{0,1\}$. Let $X^+$ denote the support of the distribution of points with positive labels, and $X^-$ denote that of the points with negative labels. Then let $H_S(f) := H[f^{-1}((0, \infty))]$ denote the \emph{support homology} of some function $f: X \to \{0,1\}$. Essentially $H_S(f)$ is homology of the set of $x$ such that $f(x) > 0$. For a binary classifier, $f$, $H_S(f)$ is roughly a characterization of how many 'holes' are in the positive decision region of $f$. We will sometimes use $\beta_n(f)$ to denote the $n$th Betti number of this support homology.  Finally let $\scriptf = \{f : X \to \{0,1\}\}$ be some family of binary classifiers on $X$.

\begin{theorem}[Homological Generalization] 
	 If $X = X^- \sqcup X^+$ and for all $f \in \scriptf$ with $H_S(f) \neq H(X^+)$, then for all $f \in \scriptf$ there exists $A \subset X^+$ so $f$ misclassifies every $x\in A.$ \label{thm:hom_princ_gen}
\end{theorem}

Essentially, Theorem \ref{thm:hom_princ_gen} says that if an architecture (a family of models $\scriptf$) is incapable of producing a certain homological complexity, then for any model using that architecture there will always be a set $A$ of true datapoints on which the model will fail. Note that the above principle holds regardless of how $f \in \scriptf$ is attained, learned or otherwise. The principle implies that no matter how well some $f$ learns to correctly classify $\scriptd$ there will always be counter examples in the true data.

In the context of architecture selection, the foregoing minimality condition significantly reduces the size of the search space by eliminating smaller architectures which cannot even express the 'holes' (persistent homology) of the data $H(\scriptd)$. This allows us to return to our original question of finding suitably expressive and generalizable architectures but in the very computable language of homological complexity: Let $\scriptf_A$ the set of all neural networks with 'architecture' $A$, then

\begin{center}
\emph{Given a dataset $\scriptd$, for which architectures $A$ does there exist  a neural network $f \in \scriptf_A$ such that $H_S(f) = H(\scriptd)$?}
\end{center}

We will resurface a contemporary theoretical view on this question, and thereafter make the first steps towards an empirical characterization of the capacity of neural architectures in the view of topology.

\subsection{Theoretical Basis for Neural Homology}

Theoretically, the homological complexity of neural network can be framed in terms of\emph{ the sum of the number of holes} expressible by certain architectures. In particular, \citet{bianchini2014complexity} gives an analysis of how the maximum sum of Betti numbers grows as $\scriptf_A$ changes. The results show that the width and activation of a fully connected architecture effect its topological expressivity to varying polynomial and exponential degrees. 

What is unclear from this analysis is how these bounds describe expressivity or learnability in terms of \emph{individual} Betti numbers. From a theoretical perspective  \citet{bianchini2014complexity} stipulated that a characterization of individual homology groups require the solution of deeper unsolved problems in algebraic topology. However, for topological complexity to be effective in architecture selection, understanding \emph{each} Betti number is essential in that it grants direct inference of architectural properties from the persistent homology of the data. Therefore we turn to an empirical characterization.

\subsection{Empirical Characterization}

To understand how the homology of data determines expressive architectures we characterize the capacities of architectures with an increasing  number of  layers and hidden units to \emph{learn} and \emph{express} homology on datasets of varying homological complexity.

\begin{figure}[t]

\begin{center}
	\includegraphics[width=0.14\textwidth]{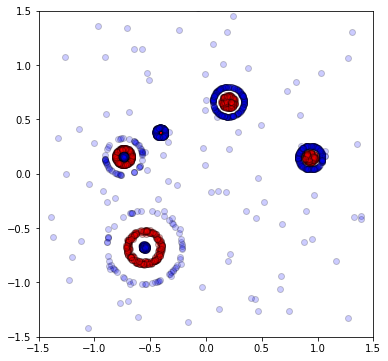}
	\includegraphics[width=0.14\textwidth]{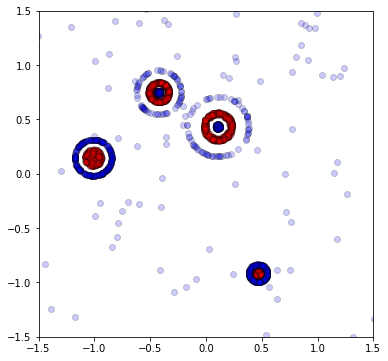}
	\includegraphics[width=0.14\textwidth]{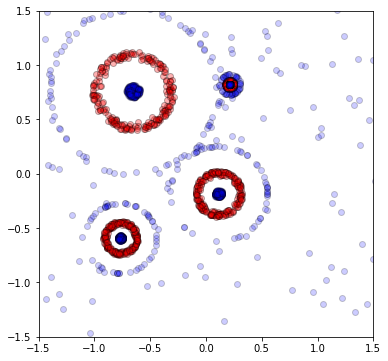}
	\includegraphics[width=0.14\textwidth]{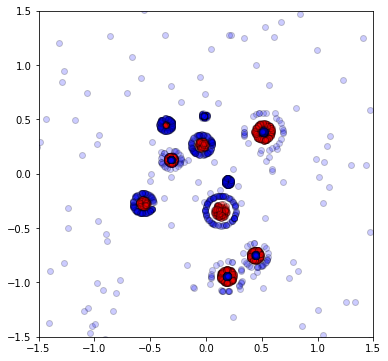}
	\includegraphics[width=0.14\textwidth]{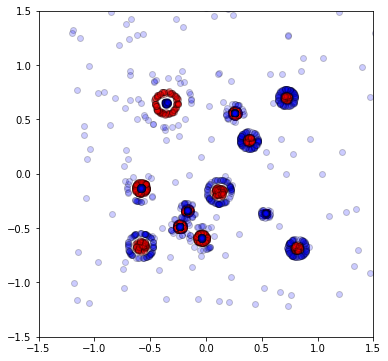}
	\includegraphics[width=0.14\textwidth]{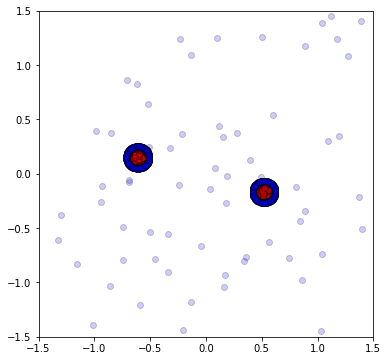}
	\caption{Scatter plots of $6$ different synthetic datasets of varying homological complexity.}
	\label{fig:datasets}
\end{center}
\end{figure}

\begin{figure*}
	\begin{center}
	\includegraphics[width=0.47\textwidth]{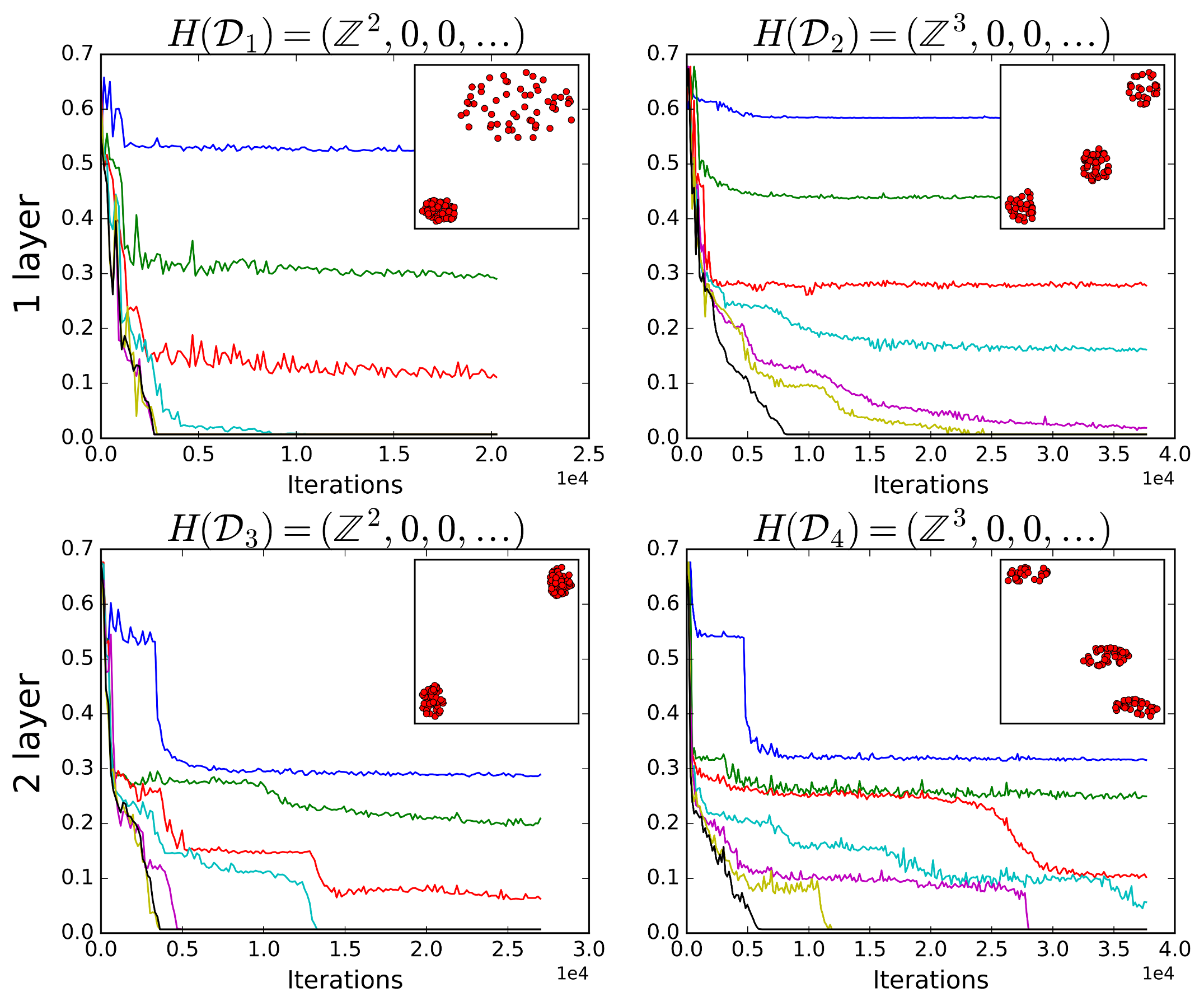}
	\includegraphics[width=0.47\textwidth]{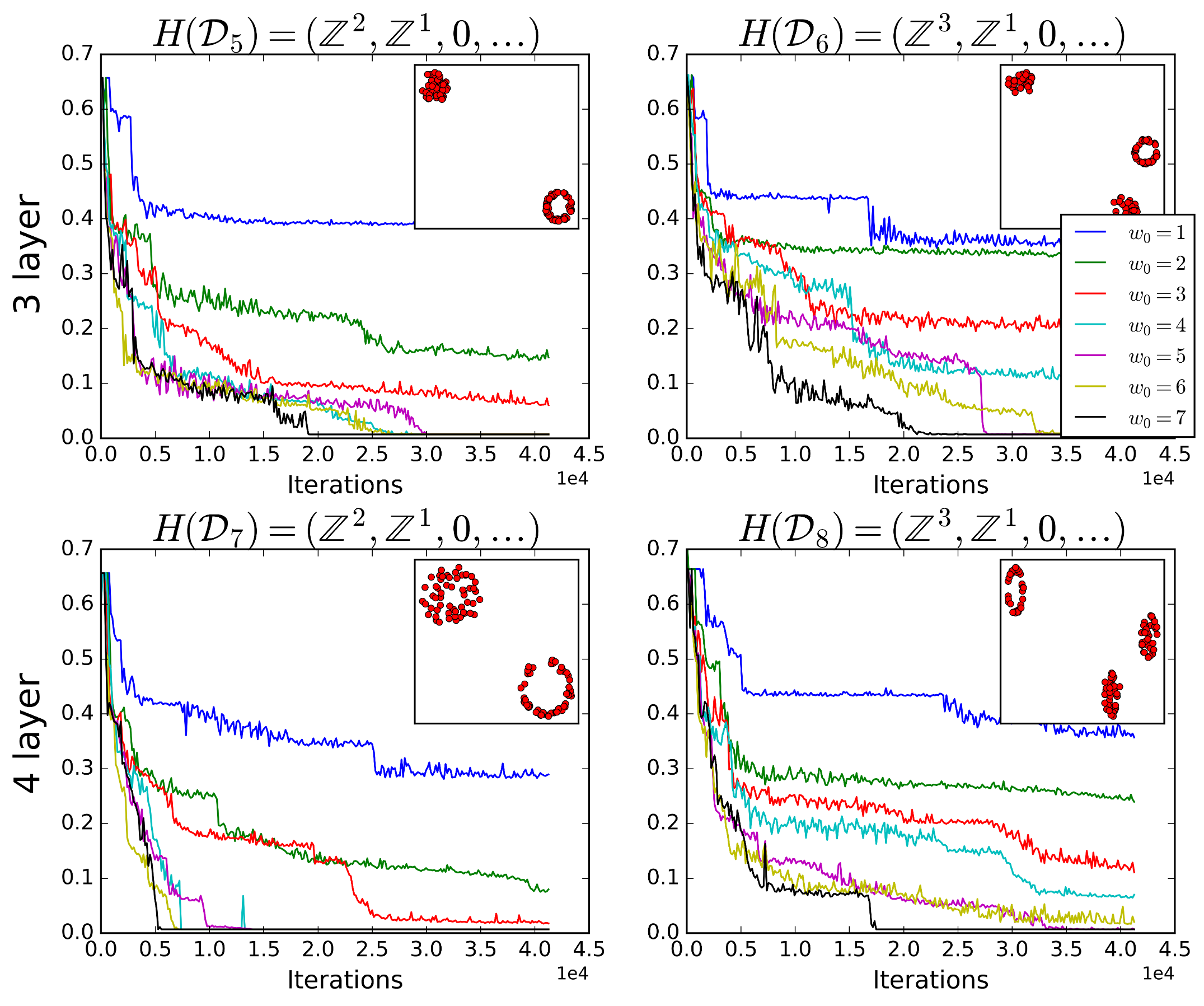}
	\end{center}

	\vspace{-10pt}
	\caption{Topological phase transitions in low dimensional neural networks as the homological complexity of the data increases. The upper right corner of each plot is a dataset on which the neural networks of increasing first layer hidden dimension are trained.  Each plot gives the minimum error for each architecture versus the number of minibatches seen.}
	\label{fig:phase}
\end{figure*}

  Restricting\footnote{Although we chose to study the low dimensional setting because it allows us to compute the persistent homology of the decision region directly, the convergence analysis extends to any number of dimensions.} our analysis to the case of $n=2$ inputs, we generate binary datasets of increasing homological complexity by sampling $ N = 5000$ points from mixtures of Unif$(\mathbb{S}^1)$ and Unif$(B^2)$, uniform random distributions on solid and empty circles with known support homologies. The homologies chosen range contiguously from $H(\scriptd) = \{\mathbb{Z}^1,0\}$ to $H(\scriptd) = \{\mathbb{Z}^{30}, \mathbb{Z}^{30}\}$ and each sampled distribution is geometrically balanced \emph{i.e.} each topological feature occupies the same order of magnitude. Additionally, margins were induced between the classes to aid learning. Examples are shown in Figure \ref{fig:datasets}.

 To characterize the difficulty of learning homologically complex data in terms of both depth and width, we consider fully connected architectures with ReLu activation functions \cite{nair2010rectified} of depth $\ell = \{1,\dots,6\}$ and width $h_l = \beta_0(\scriptd)$ when $1 \leq l \leq \ell$ and $h_l \in \{1, \dots, 500\}$ when $l =0$. We will denote individual architectures by the pair $(\ell, h_0)$. We vary the number of hidden units in the first layer, as they form a half-space basis for the decision region. The weights of each architecture are initialized to samples from a normal distribution $\scriptn(0,\frac{1}{\beta_0})$ with variance respecting the scale of each synthetic dataset. For each homology we take several datasets sampled from the foregoing procedure and optimize $100$ initializations of each architecture against the standard cross-entropy loss. To minimize the objective we use the Adam optimizer \cite{kingma2014adam} with a fixed learning rate of $0.01$ and an increasing batch size schedule \cite{smith2017don}.

 We compare each architectures average and best performance by measuring misclassification error over the course of training and homological expressivity at the end of training. The latter quantity, given by 
\begin{equation*}
	E_H^p(f, \scriptd) = \min\left\{\frac{\beta_p(f)}{\beta_p(\scriptd)},1\right\},
\end{equation*}
measures the capacity of a model to exhibit the true homology of the data. We compute the homology of individual decision regions by constructing a filtration on Heaviside step function of the difference of the outputs, yielding a persistence diagram with the exact homological components of the decision regions. The results are summarized in Figures \ref{fig:phase}-\ref{fig:homexistence}

 The resulting convergence analysis indicates that neural networks exhibit a statistically significant \emph{topological phase transition} during learning which depends directly on the homological complexity of the data. For any dataset and any random homeomorphism applied thereto, the best error of architectures with $\ell$ layers and $h$ hidden units (on the first layer)  is \emph{strictly} 
limited in magnitude and convergence time by $h_{phase}$. For example in Figure \ref{fig:phase}, $\ell = 3$ layer neural networks fail to converge for $h < h_{phase} = 4$ on datasets with homology $H(\scriptd_5) = (\mathbb{Z}^2, \mathbb{Z}^1, \dots)$. 

More generally, homological complexity directly effects the efficacy of optimization in neural networks. As shown in Figure \ref{fig:hom_converge}, taking any increasing progression of homologies against the average convergence time of a class of architectures yields an approximately monotonic relationship; in this case, convergence time at $h_{phase}$ increases with increasing $\beta_p(\scriptd)$, and convergence time at $h > h_{phase}$ decreases with fixed $\beta_p(\scriptd)$. A broader analysis for a varying number of layers is given in the appendix.

\begin{figure}[t!hb]
	\begin{center}
		\includegraphics[width=0.46\textwidth]{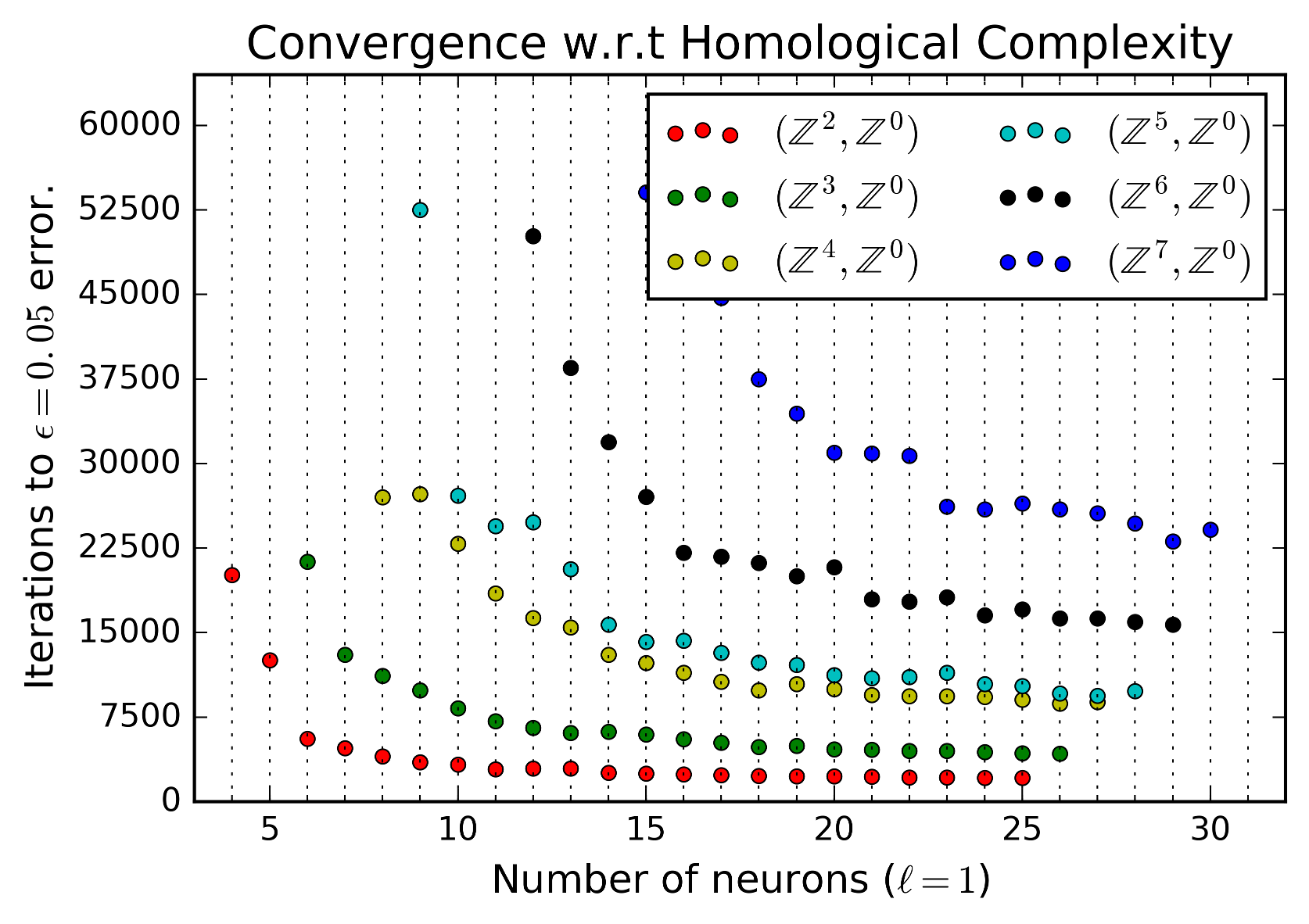}
	\end{center}
	\vspace{-10pt}
	\caption{A scatter plot of the number of iterations required for single-layer architectures of varying hidden dimension to converge to 5\% misclassification error. The colors of each point denote the topological complexity of the data on which the networks were trained. Note the emergence of monotonic bands. Multilayer plots given in the appendix look similar.}
	\label{fig:hom_converge}
\end{figure}

\begin{figure}[t!hb]
	\begin{center}
		\includegraphics[width=0.47\textwidth]{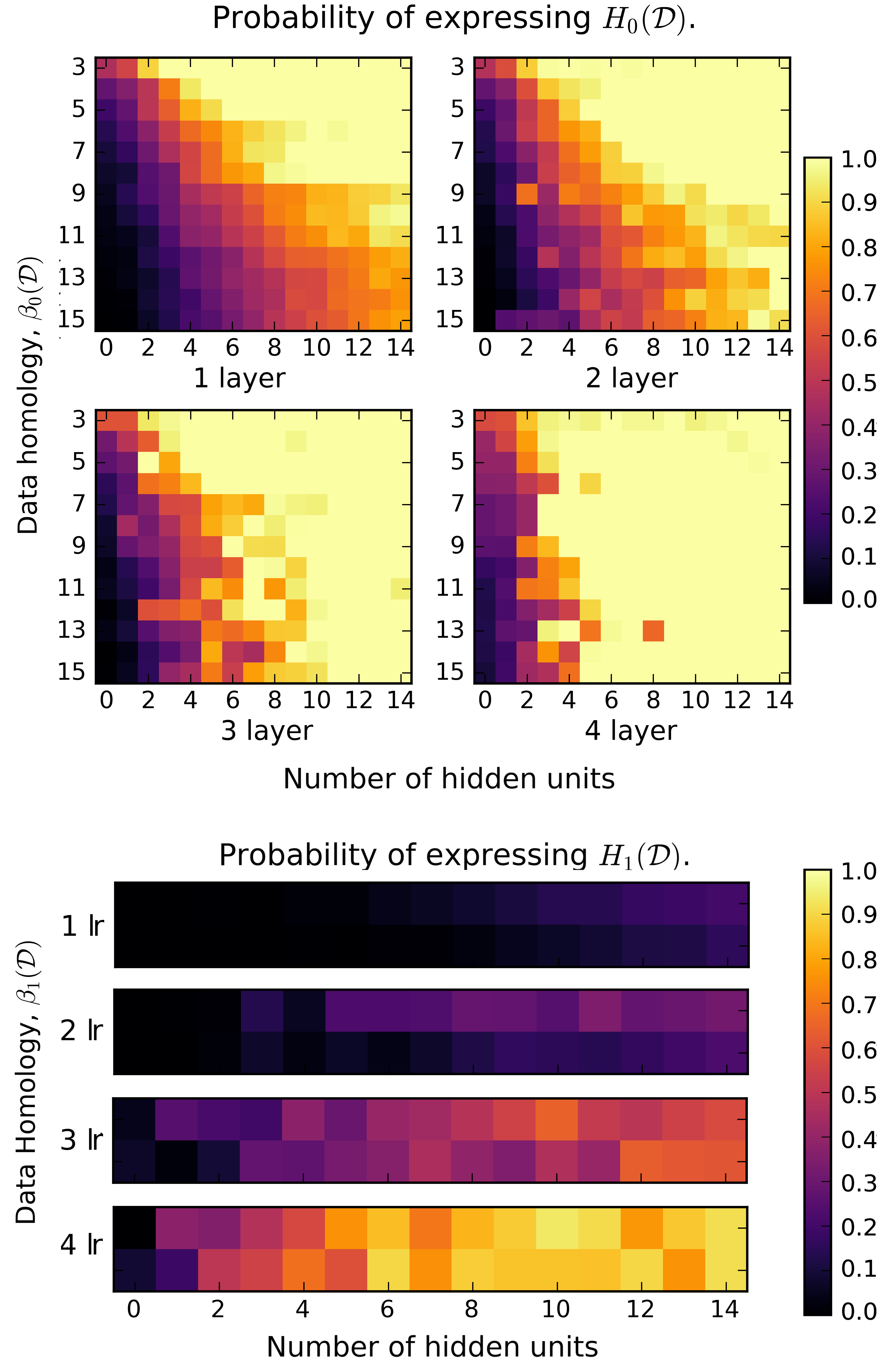}
	\end{center}
	\vspace{-10pt}
	\caption{A table of estimated probabilities of different neural architectures to express certain homological features of the data after training. Top: the probabilities of express homologies with increasing $\beta_0$ as a function of layers and neurons. Bottom: The probabilities of expressing $\beta_1 \in \{1,2\}$ as a function of layers and neurons.}
	\label{fig:homexistence}
\end{figure}

Returning to the initial question of architecture selection, the analysis of empirical estimation of $E_H^p(f, \scriptd)$  provides the first complete probabilistic picture of the homological expressivity of neural architectures. For architectures with $\ell \in \{1,2,3,4\}$ and $h_0 \in \{0,\dots,30\}$ Figure \ref{fig:homexistence} displays the estimated probability that $(\ell, h_0)$ expresses the homology of the decision region after training. Specifically, Figure \ref{fig:homexistence}(top) indicates that, for $\ell = 1$ hidden layer neural networks, $\max \beta_0(f))$ is clearly $\Omega(h_0)$. Examining $\ell =2,3,4$ in Figure \ref{fig:homexistence}(top), we conjecture\footnote{We note that $\beta_0(f)$ does not depend on $\beta_1(\scriptd)$ in the experiment for datasets with $\beta_1(\scriptd) > \beta_0(\scriptd)$ were not generated.} that as $\ell$ increases 
\begin{equation*}
\max_{f \in F_A} \beta_0(f) \in \Omega(h_0^{\ell}))   ,
\end{equation*} by application of Theorem 3.1 to the expressivity estimates. Therefore
 \begin{equation}
 h_{phase} \geq C\sqrt[\ell]{\beta_0(\scriptd)}.\label{eq:bound}
 \end{equation} Likewise, in Figure \ref{fig:homexistence}(bottom), each horizontal matrix gives the probability of expressing $\beta_1(\scriptd) \in \{1,2\}$ for each layer. This result indicates that higher order homology is extremely difficult to learn in the single-layer case, and as $\ell \to \infty$, $\max_{f \in f_A} \beta_1(f) \to n = 2$, the input dimension. 

The importance of the foregoing empirical characterization for architecture selection is two-fold. First, by analyzing the effects of homological complexity on the optimization of different architectures, we were able to conjecture probabilistic bounds on the \emph{learnable} homological capacity of neural networks. Thus, predictions of minimal architectures using those bounds are sufficient enough to {learn} data homology up to homeomorphism.  Second, the analysis of individual Betti numbers enables data-first architecture selection using persistent homology.

\section{Topological Architecture Selection}

\begin{figure*}
	\includegraphics[width=0.49\textwidth]{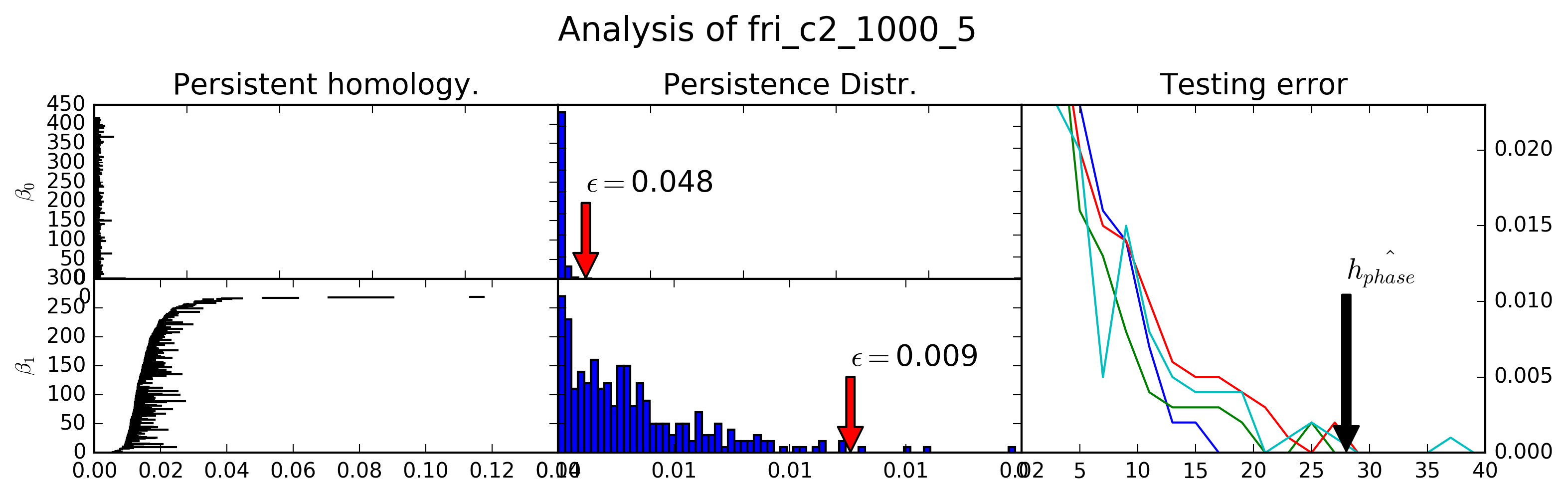}
	\includegraphics[width=0.49\textwidth]{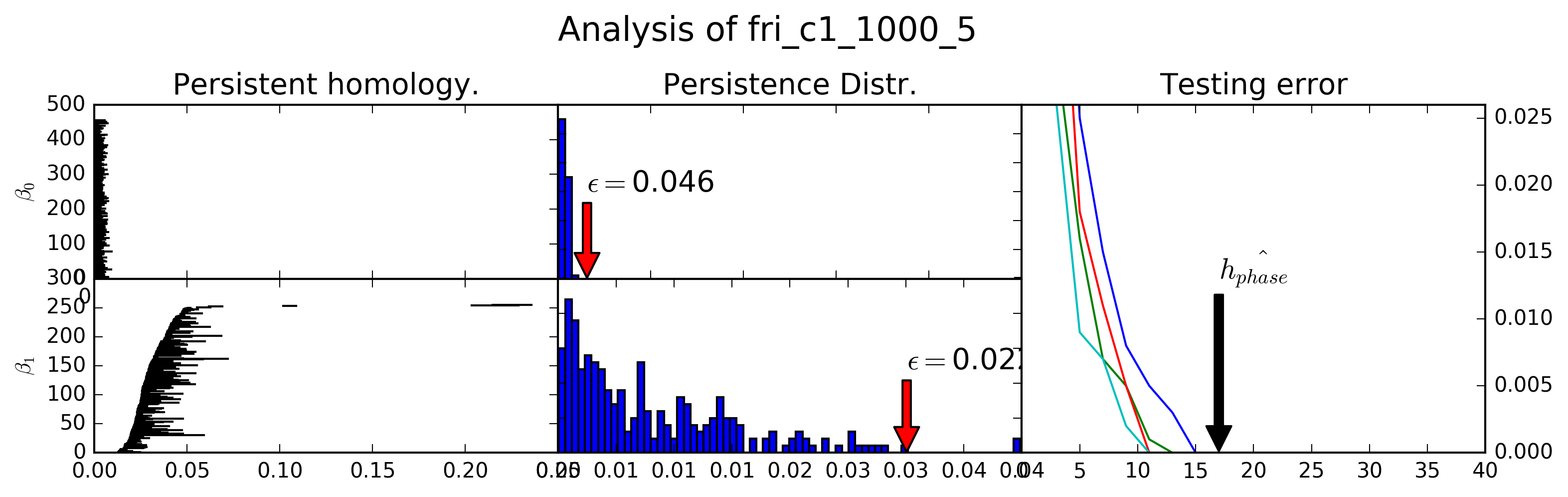}
	\includegraphics[width=0.49\textwidth]{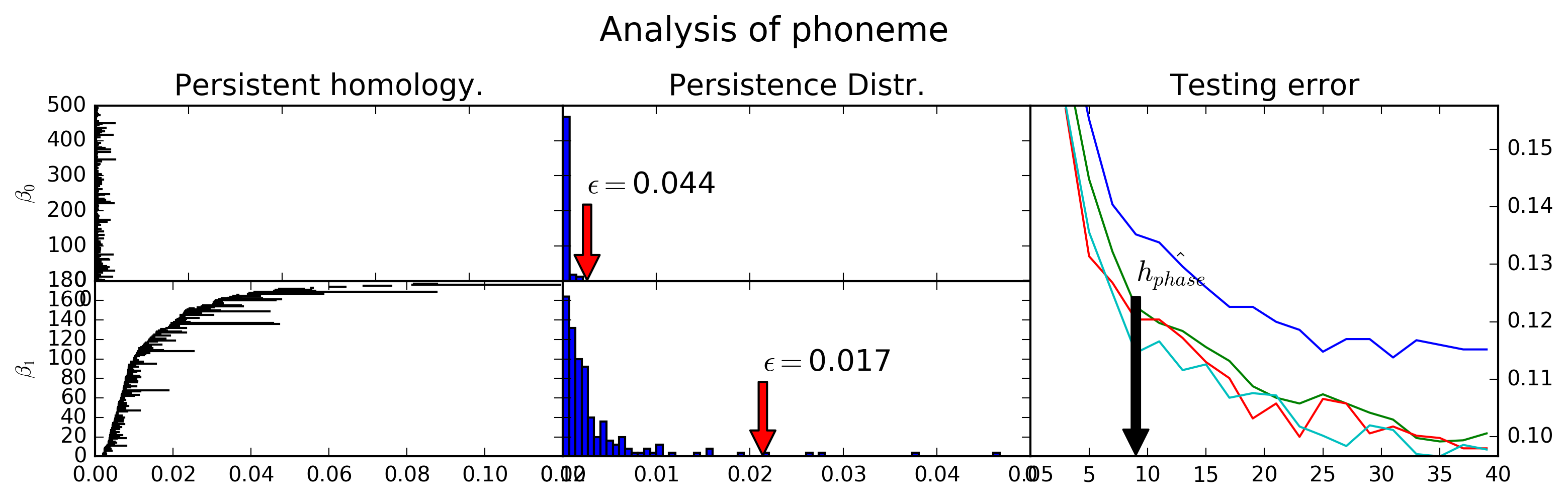}
	\includegraphics[width=0.49\textwidth]{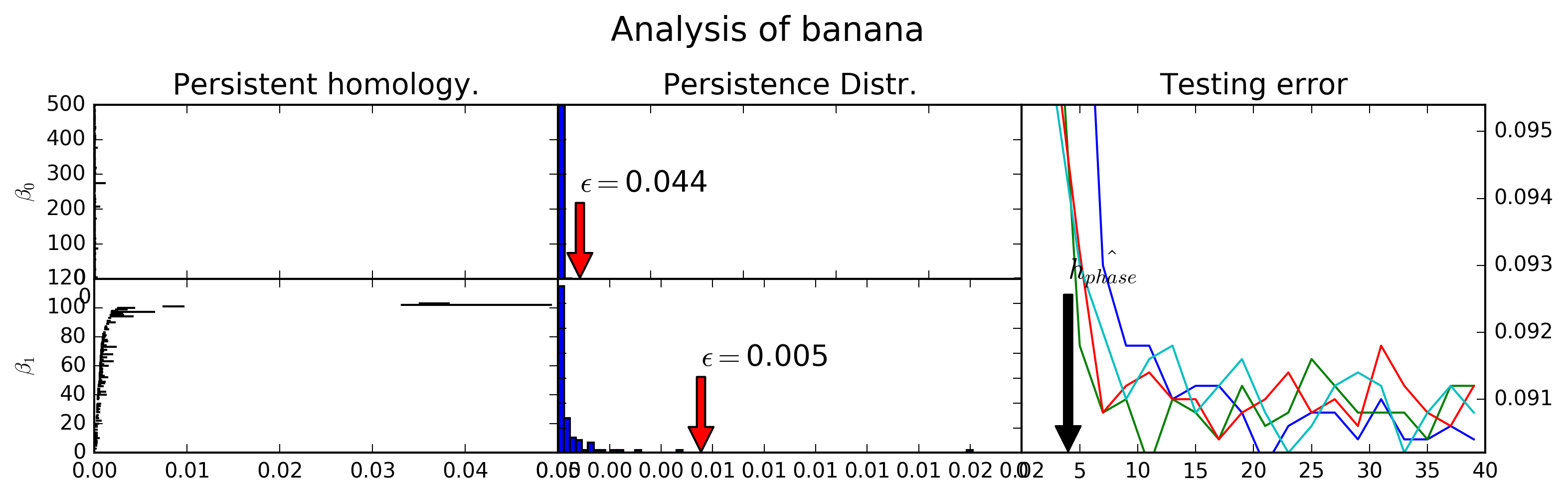}
	\label{fig:top_arch_sel}
	\caption{Topological architecture selection applied to four different datasets. The persistent homology, the histogram of topological lifespans, and the predicted $\hat{h}_{phase}$ are indicated. For each test error plot, the best performance of one (blue), two (green), three (red), and four (light blue) layer neural networks are given in terms of the number of hidden units on the first layer.}
\end{figure*}

We have thus far demonstrated the discriminatory power of homological complexity in determining the expressivity of architectures. However, for homological complexity to have any practical use in architecture selection, it must be computable on real data, and more generally real data must have non-trivial homology. In the following section we present a method for relating persistent homology of any dataset to a minimally expressive architecture predicted by the foregoing empirical characterization, and then we experimentally validate our method on several datasets.

Topological architecture selection is comprised of three steps: given a dataset, compute its persistent homology; determine an appropriate scale at which to accept topological features as pertinent to the learning problem; and infer a lower-bound on $h_{phase}$ from the topological features at or above the decided scale. 

The extraction of static homology from persistence homology, while aesthetically valid \cite{carlsson2008local}, is ill-posed in many cases. For the purposes of architecture selection, however, exact reconstruction of the original homology is not necessary. Given a persistence diagram, $D_p$ containing the births $b_i$ and deaths $d_i$ of features in $H_p(\scriptd)$, let~$\epsilon$ be given and consider all $\alpha_i = (b_i, d_i)$ such that when $|d_i - b_i| > \epsilon$ we assume $\alpha$ to be a topological component of the real space. Then the resulting homologies form a filtration $H_p^{\epsilon_1}(\scriptd) \subset H_p^{\epsilon_2}$ for $\epsilon_1, \epsilon_2 \in \mathbb{R}^+$. If $\epsilon$ is chosen such that certain topologically noisy features \cite{fasy2014confidence} are included in the estimate of $h_{phase}$, then at worst the architecture is overparameterized, but still learns. If $\epsilon$ is such that the estimated $h_{phase}$ is underrepresentitive of the topological features of a space, then at worst, the architecture is underparameterized but potentially close to $h_{phase}.$  As either case yields the solution or a plausibly useful seed for other algorithm selection algorithms \cite{45826,feurer2015initializing}, we adopt the this static instantiation of persistence homology.

In order to select an architecture $(\ell, h_0)$ lower-bounding $h_{phase}$, we restrict the analysis to the case of $\ell =1$ and regress a multilinear model training on pairs
\begin{equation*}
	(b_0, b_1) \mapsto \arg \min_m \ E_H^p(f_{m}, \scriptd)\geq 1, \beta_*(\scriptd) = (b_0, b_1)
\end{equation*}
over all $m$ hidden unit single layer neural networks $f_m$  and synthetic datasets of known homology $\scriptd$ from our previous experiments. The resultant discretization of the model gives a lower-bound estimate after applying the bound from \eqref{eq:bound}:
\begin{equation}
\label{eq:estimate}
 \hat{h}_{phase}(\beta_0, \beta_1) \geq \beta_1 C\sqrt[\ell]{(\beta_0)}  + 2
 \end{equation}
  Estimating this lower-bound is at the core of neural homology theory and is the subject of substantial future theoretical and empirical work.

\subsection{Results}

In order to validate the approach, we applied topological architecture selection to several binary datasets from the OpenML dataset repository \cite{OpenML2013}: \texttt{fri\_c}, 
\texttt{balance-scale}, \texttt{banana}, \texttt{phoneme}, and \texttt{delta\_ailerons}.
 We compute persistent homology of each of the two labeled classes therein and accept topological features with lifespans greater than two standard deviations from the mean for each homological dimension. We then estimated a lowerbound on $h_{phase}$ for single hidden layer neural networks using \ref{eq:estimate}. Finally we trained $100$ neural networks for each architecture $(\ell, h_0)$ where $h_0 \in\{1,3,\dots,99\}$ and $\ell \in \{1,\dots,4\}$. During training we record the minimum average error for each $h_0$ and compare this with the estimate $\hat{h}_{phase}.$ The results are summarized in \ref{fig:top_arch_sel}.

These preliminary findings indicate that the empirically derived estimate of $h_{phase}$ provides a strong starting point for architecture selection, as the minimum error at $\hat{h}_{phase}$ is near zero in every training instance. Although the estimate is given only in terms of $0$ and $1$ dimensional homology of the data, it still performed well for higher dimensional datasets such as \texttt{phoneme} and \texttt{fri\_c*}. In failure cases, choice of $\epsilon$ greatly affected the predicted $h_{phase}$, and thus it is imperative that more adaptive topological selection schemes be investigated.

While our analysis and characterization is given for the the decision regions of individual classes in a dataset, it is plausible that the true decision boundary is topologically simple despite the complexity of the classes. Although we did not directly characterize neural network by the topology of their decision boundaries,  recent work by \citet{varshney2015persistent} provides an exact method for computing the persistent homology of a decision boundary between any number of classes in a dataset. It is the subject of future work to provide a statistical foundation for \cite{varshney2015persistent} and then reanalyze the homological capacity of neural networks in this context.

\subsection{Homological Complexity of Real Data}

In addition to testing topological measures of complexity in the setting of architecture selection, we verify that common machine learning benchmark datasets have non-trivial homology and the computation thereof is tractable.

\textbf{CIFAR-10.} We compute the persistent homology of several classes of CIFAR-10  using the Python library Dionysus. Current algorithms for persistent homology do not deal well with high dimensional data, so we embed the entire dataset in $\mathbb{R}^3$ using local linear embedding (LLE; \cite{saul2000introduction})  with $K = 120$ neighbors. We note that an embedding of any dataset to a lower dimensional subspace with small enough error roughly preserves the static homology of the support of the data distribution distribution by Theorem \ref{thm:top_to_hom}. After embedding the dataset, we take a sample of $1000$ points from example class 'car' and build a persistent filtration by constructing a Vietoris-Rips complex on the data. The resulting complex has $20833750$ simplices and took $4.3$ min. to generate. Finally, computation of the persistence diagram shown in Figure \ref{fig:dionysus_persis} took $8.4$ min. locked to a single thread on a Intel Core i7 processor. The one-time cost of computing persistent homology could easily augment any neural architecture search.

 Although we only give an analysis of dimension $2$ topological features--and there is certainly higher dimensional homological information in CIFAR-10--the persistence barcode diagram is rich with different components in both $H_0(\scriptd)$ and $H_1(\scriptd).$ Intuitively, CIFAR contains pictures of cars rotated across a range of different orientations and this is exhibited in the homology. In particular, several holes are born and die in the range $\epsilon \in [0.15, 0.375]$ and one large loop from $\epsilon \in [0.625, 0.82]$. 
\begin{figure}[t]
	\begin{center}
		\includegraphics[width=0.47\textwidth]{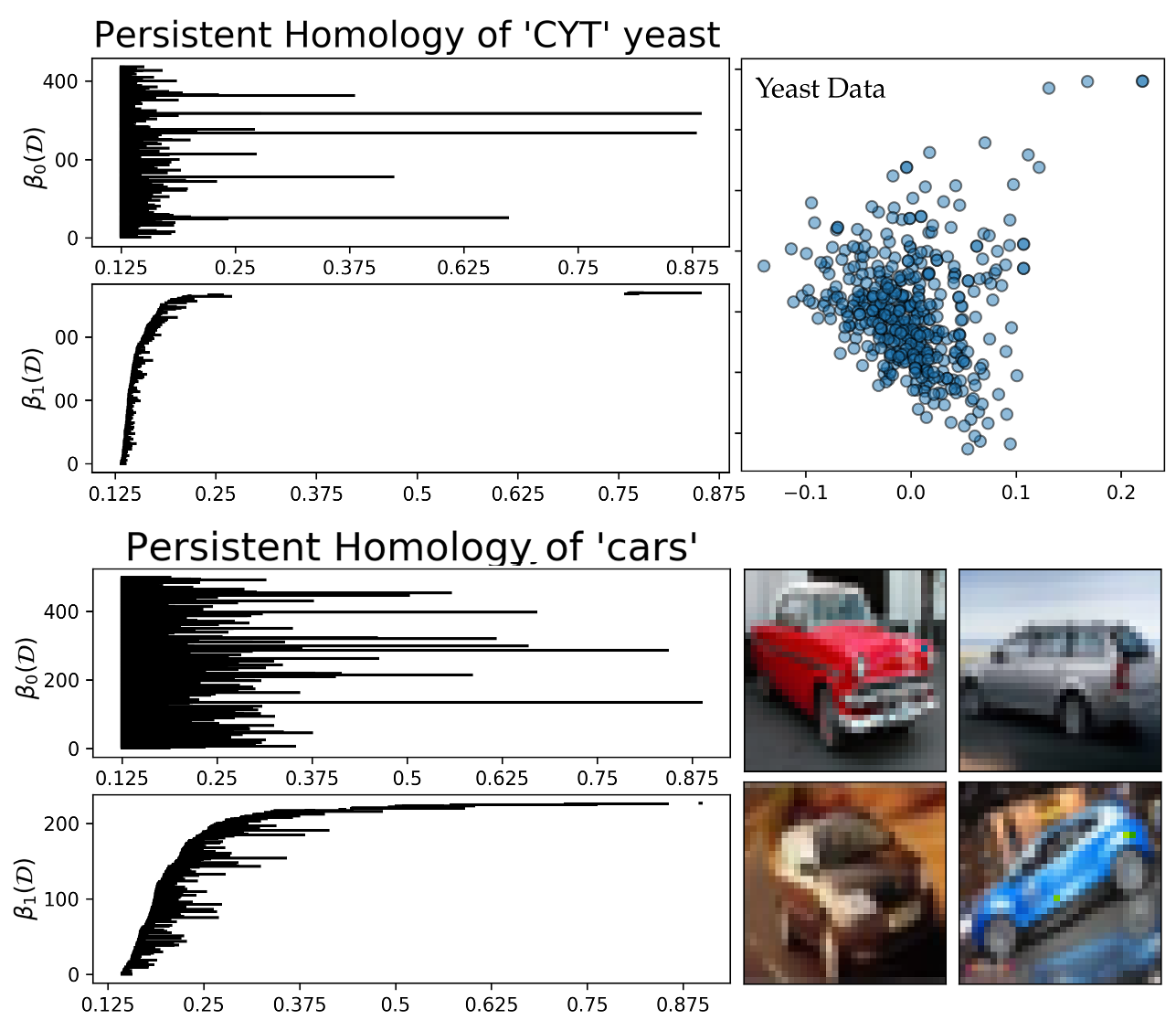}
		\caption{The persistent homology barcodes of classes in the CIFAR-10 Datasets; The barcode for the dimensions $0$ and $1$ for the 'cars' class along side different samples. Note how different orientations are shown.}\label{fig:dionysus_persis}
	\end{center}
\end{figure}

\textbf{UCI Datasets.} We further compute the homology of three low dimensional UCI datasets and attempt to assert the of non-trivial , $h_{phase}$. Specifically, we compute the persistent homology of the majority classes in the Yeast Protein Localization Sites, UCI Ecoli Protein Localization Sites, and HTRU2 datasets. For these datasets no dimensionality reduction was used. In Figure \ref{fig:dionysus_persis}(left), the persistence barcode exhibits two  separate significant loops (holes) at $\epsilon \in [0.19, 0.31]$ and $\epsilon \in [0.76,0.85]$, as well as two major connected components in $\beta_0(\scriptd).$ The Other persistence diagrams are relegated to the appendix.

\section{Related Work}

We will place this work in the context of  deep learning theory as it relates to expressivity. Since the seminal work of \citet{cybenko1989approximation} which established standard universal approximation results for neural networks, many researchers have attempted to understand the expressivity of certain neural architectures. \citet{pascanu2013number} and \citet{mackay2003information} provided the first analysis relating the depth and width of architectures to the complexity of the sublevel sets they can express. Motivated therefrom, \citet{bianchini2014complexity} expressed this theme in the language of Pfefferian functions, thereby bounding the sum of Betti numbers expressed by sublevel sets. Finally \citet{guss2016deep} gave an account of how topological assumptions on the input data lead to optimally expressive architectures. In parallel, \citet{eldan2016power} presented the first analytical minimality result in expressivity theory; that is, the authors show that there are simple functions that cannot be expressed by two layer neural networks without exponential dependence on input dimension.  This work spurred the work of \citet{poole2016exponential}, \citet{raghu2016expressive} which reframed expressivity in a differential geometric lens.

Our work presents the first method to derive expressivity results empirically. Our topological viewpoint sits dually with its differential geometric counterpart, and in conjunction with the work of \cite{poole2016exponential} and \cite{bianchini2014complexity}. This duality implies that when topological expression is not possible, exponential differential expressivity allows networks to bypass homological constraints at the cost of adversarial sets. Furthermore, our work opens a practical connection between the foregoing theory on neural expressivity and architecture selection, with the potential to substantially improve neural architecture search \cite{45826} by directly computing the capacities of different architectures.

\section{Conclusion}

Architectural power is closely related to the algebraic topology of decision regions. In this work we distilled neural network expressivity into an empirical question of the generalization capabilities of architectures with respect to the homological complexity of learning problems. This view allowed us to provide an empirical method for developing tighter characterizations on the the capacity of different architectures in addition to a principled approach to guiding architecture selection by computation of persistent homology on real data.

There are several potential avenues of future research in using homological complexity to better understand neural architectures. First, a full characterization of neural networks with convolutional linearities and state-of-the-art topologies is a crucial next step. Our empirical results suggest that there are exact formulas describing the of power of neural networks to express decision boundaries with certain properties. Future theoretical work in determining these forms would significantly increase the efficiency and power of neural architecture search, constraining the search space by the persistent homology of the data. Additionally, we intend on studying how the topological complexity of data changes as it is propagated through deeper architectures.

\section{Acknowledgements}

We thank Greg Yang, Larry Wasserman, Guy Wilson, Dawn Song and Peter Bartlett for useful discussions and insights. 

%

\bibliographystyle{arxiv2018}
\bibliography{ntop}

\begin{thebibliography}{34}
\providecommand{\natexlab}[1]{#1}
\providecommand{\url}[1]{\texttt{#1}}
\expandafter\ifx\csname urlstyle\endcsname\relax
  \providecommand{\doi}[1]{doi: #1}\else
  \providecommand{\doi}{doi: \begingroup \urlstyle{rm}\Url}\fi

\bibitem[Betti(1872)]{betti1872nuovo}
Betti, E.
\newblock Il nuovo cimento.
\newblock \emph{Series}, 2:\penalty0 7, 1872.

\bibitem[Bianchini et~al.(2014)]{bianchini2014complexity}
Bianchini, Monica et~al.
\newblock On the complexity of shallow and deep neural network classifiers.
\newblock In \emph{ESANN}, 2014.

\bibitem[Bredon(2013)]{bredon2013topology}
Bredon, Glen~E.
\newblock \emph{Topology and geometry}, volume 139.
\newblock Springer Science \& Business Media, 2013.

\bibitem[Carlsson et~al.(2008)Carlsson, Ishkhanov, De~Silva, and
  Zomorodian]{carlsson2008local}
Carlsson, Gunnar, Ishkhanov, Tigran, De~Silva, Vin, and Zomorodian, Afra.
\newblock On the local behavior of spaces of natural images.
\newblock \emph{International journal of computer vision}, 76\penalty0
  (1):\penalty0 1--12, 2008.

\bibitem[Cybenko(1989)]{cybenko1989approximation}
Cybenko, George.
\newblock Approximation by superpositions of a sigmoidal function.
\newblock \emph{Mathematics of Control, Signals, and Systems (MCSS)},
  2\penalty0 (4):\penalty0 303--314, 1989.

\bibitem[Daniely et~al.(2016)Daniely, Frostig, and Singer]{daniely2016toward}
Daniely, Amit, Frostig, Roy, and Singer, Yoram.
\newblock Toward deeper understanding of neural networks: The power of
  initialization and a dual view on expressivity.
\newblock In \emph{Advances In Neural Information Processing Systems}, pp.\
  2253--2261, 2016.

\bibitem[Dey et~al.(1998)Dey, Edelsbrunner, and Guha]{dey1998computational}
Dey, TK, Edelsbrunner, H, and Guha, S.
\newblock Computational topology, invited paper in advances in discrete and
  computational geometry, eds. b. chazelle, je goodmann and r. pollack.
\newblock \emph{Contemporary Mathematics, AMS}, 1998.

\bibitem[Eldan \& Shamir(2016)Eldan and Shamir]{eldan2016power}
Eldan, Ronen and Shamir, Ohad.
\newblock The power of depth for feedforward neural networks.
\newblock In \emph{Conference on Learning Theory}, pp.\  907--940, 2016.

\bibitem[Fasy et~al.(2014)Fasy, Lecci, Rinaldo, Wasserman, Balakrishnan, Singh,
  et~al.]{fasy2014confidence}
Fasy, Brittany~Terese, Lecci, Fabrizio, Rinaldo, Alessandro, Wasserman, Larry,
  Balakrishnan, Sivaraman, Singh, Aarti, et~al.
\newblock Confidence sets for persistence diagrams.
\newblock \emph{The Annals of Statistics}, 42\penalty0 (6):\penalty0
  2301--2339, 2014.

\bibitem[Fernando et~al.(2017)Fernando, Banarse, Blundell, Zwols, Ha, Rusu,
  Pritzel, and Wierstra]{DBLP:journals/corr/FernandoBBZHRPW17}
Fernando, Chrisantha, Banarse, Dylan, Blundell, Charles, Zwols, Yori, Ha,
  David, Rusu, Andrei~A., Pritzel, Alexander, and Wierstra, Daan.
\newblock Pathnet: Evolution channels gradient descent in super neural
  networks.
\newblock \emph{CoRR}, abs/1701.08734, 2017.
\newblock URL \url{http://arxiv.org/abs/1701.08734}.

\bibitem[Feurer et~al.(2015)Feurer, Springenberg, and
  Hutter]{feurer2015initializing}
Feurer, Matthias, Springenberg, Jost~Tobias, and Hutter, Frank.
\newblock Initializing bayesian hyperparameter optimization via meta-learning.
\newblock In \emph{AAAI}, pp.\  1128--1135, 2015.

\bibitem[Goodfellow et~al.(2016)Goodfellow, Bengio, and
  Courville]{Goodfellow-et-al-2016}
Goodfellow, Ian, Bengio, Yoshua, and Courville, Aaron.
\newblock \emph{Deep Learning}.
\newblock MIT Press, 2016.
\newblock \url{http://www.deeplearningbook.org}.

\bibitem[Guss(2016)]{guss2016deep}
Guss, William~H.
\newblock Deep function machines: Generalized neural networks for topological
  layer expression.
\newblock \emph{arXiv preprint arXiv:1612.04799}, 2016.

\bibitem[He et~al.(2015)He, Zhang, Ren, and Sun]{DBLP:journals/corr/HeZRS15}
He, Kaiming, Zhang, Xiangyu, Ren, Shaoqing, and Sun, Jian.
\newblock Deep residual learning for image recognition.
\newblock \emph{CoRR}, abs/1512.03385, 2015.
\newblock URL \url{http://arxiv.org/abs/1512.03385}.

\bibitem[Hinton et~al.(2012)Hinton, Deng, Yu, Dahl, Mohamed, Jaitly, Senior,
  Vanhoucke, Nguyen, Sainath, et~al.]{hinton2012deep}
Hinton, Geoffrey, Deng, Li, Yu, Dong, Dahl, George~E, Mohamed, Abdel-rahman,
  Jaitly, Navdeep, Senior, Andrew, Vanhoucke, Vincent, Nguyen, Patrick,
  Sainath, Tara~N, et~al.
\newblock Deep neural networks for acoustic modeling in speech recognition: The
  shared views of four research groups.
\newblock \emph{IEEE Signal Processing Magazine}, 29\penalty0 (6):\penalty0
  82--97, 2012.

\bibitem[Kingma \& Ba(2014)Kingma and Ba]{kingma2014adam}
Kingma, Diederik~P and Ba, Jimmy.
\newblock Adam: A method for stochastic optimization.
\newblock \emph{arXiv preprint arXiv:1412.6980}, 2014.

\bibitem[Krizhevsky et~al.(2012)Krizhevsky, Sutskever, and
  Hinton]{krizhevsky2012imagenet}
Krizhevsky, Alex, Sutskever, Ilya, and Hinton, Geoffrey~E.
\newblock Imagenet classification with deep convolutional neural networks.
\newblock In \emph{Advances in neural information processing systems}, pp.\
  1097--1105, 2012.

\bibitem[MacKay(2003)]{mackay2003information}
MacKay, David~JC.
\newblock \emph{Information theory, inference and learning algorithms}.
\newblock Cambridge university press, 2003.

\bibitem[Mnih et~al.(2013)Mnih, Kavukcuoglu, Silver, Graves, Antonoglou,
  Wierstra, and Riedmiller]{mnih2013playing}
Mnih, Volodymyr, Kavukcuoglu, Koray, Silver, David, Graves, Alex, Antonoglou,
  Ioannis, Wierstra, Daan, and Riedmiller, Martin.
\newblock Playing atari with deep reinforcement learning.
\newblock \emph{arXiv preprint arXiv:1312.5602}, 2013.

\bibitem[Nair \& Hinton(2010)Nair and Hinton]{nair2010rectified}
Nair, Vinod and Hinton, Geoffrey~E.
\newblock Rectified linear units improve restricted boltzmann machines.
\newblock In \emph{Proceedings of the 27th international conference on machine
  learning (ICML-10)}, pp.\  807--814, 2010.

\bibitem[Pascanu et~al.(2013)Pascanu, Montufar, and Bengio]{pascanu2013number}
Pascanu, Razvan, Montufar, Guido, and Bengio, Yoshua.
\newblock On the number of response regions of deep feed forward networks with
  piece-wise linear activations.
\newblock \emph{arXiv preprint arXiv:1312.6098}, 2013.

\bibitem[Poole et~al.(2016)Poole, Lahiri, Raghu, Sohl-Dickstein, and
  Ganguli]{poole2016exponential}
Poole, Ben, Lahiri, Subhaneil, Raghu, Maithreyi, Sohl-Dickstein, Jascha, and
  Ganguli, Surya.
\newblock Exponential expressivity in deep neural networks through transient
  chaos.
\newblock In \emph{Advances In Neural Information Processing Systems}, pp.\
  3360--3368, 2016.

\bibitem[Raghu et~al.(2016)Raghu, Poole, Kleinberg, Ganguli, and
  Sohl-Dickstein]{raghu2016expressive}
Raghu, Maithra, Poole, Ben, Kleinberg, Jon, Ganguli, Surya, and Sohl-Dickstein,
  Jascha.
\newblock On the expressive power of deep neural networks.
\newblock \emph{arXiv preprint arXiv:1606.05336}, 2016.

\bibitem[Saul \& Roweis(2000)Saul and Roweis]{saul2000introduction}
Saul, Lawrence~K and Roweis, Sam~T.
\newblock An introduction to locally linear embedding.
\newblock \emph{unpublished. Available at: http://www. cs. toronto. edu/\~{}
  roweis/lle/publications. html}, 2000.

\bibitem[Saxena \& Verbeek(2016)Saxena and
  Verbeek]{DBLP:journals/corr/SaxenaV16}
Saxena, Shreyas and Verbeek, Jakob.
\newblock Convolutional neural fabrics.
\newblock \emph{CoRR}, abs/1606.02492, 2016.
\newblock URL \url{http://arxiv.org/abs/1606.02492}.

\bibitem[Simonyan \& Zisserman(2014)Simonyan and
  Zisserman]{DBLP:journals/corr/SimonyanZ14a}
Simonyan, Karen and Zisserman, Andrew.
\newblock Very deep convolutional networks for large-scale image recognition.
\newblock \emph{CoRR}, abs/1409.1556, 2014.
\newblock URL \url{http://arxiv.org/abs/1409.1556}.

\bibitem[Smith et~al.(2017)Smith, Kindermans, and Le]{smith2017don}
Smith, Samuel~L, Kindermans, Pieter-Jan, and Le, Quoc~V.
\newblock Don't decay the learning rate, increase the batch size.
\newblock \emph{arXiv preprint arXiv:1711.00489}, 2017.

\bibitem[Szegedy et~al.(2014)Szegedy, Liu, Jia, Sermanet, Reed, Anguelov,
  Erhan, Vanhoucke, and Rabinovich]{DBLP:journals/corr/SzegedyLJSRAEVR14}
Szegedy, Christian, Liu, Wei, Jia, Yangqing, Sermanet, Pierre, Reed, Scott~E.,
  Anguelov, Dragomir, Erhan, Dumitru, Vanhoucke, Vincent, and Rabinovich,
  Andrew.
\newblock Going deeper with convolutions.
\newblock \emph{CoRR}, abs/1409.4842, 2014.
\newblock URL \url{http://arxiv.org/abs/1409.4842}.

\bibitem[Topaz et~al.(2015)Topaz, Ziegelmeier, and
  Halverson]{topaz2015topological}
Topaz, Chad~M, Ziegelmeier, Lori, and Halverson, Tom.
\newblock Topological data analysis of biological aggregation models.
\newblock \emph{PloS one}, 10\penalty0 (5):\penalty0 e0126383, 2015.

\bibitem[Vanschoren et~al.(2013)Vanschoren, van Rijn, Bischl, and
  Torgo]{OpenML2013}
Vanschoren, Joaquin, van Rijn, Jan~N., Bischl, Bernd, and Torgo, Luis.
\newblock Openml: Networked science in machine learning.
\newblock \emph{SIGKDD Explorations}, 15\penalty0 (2):\penalty0 49--60, 2013.
\newblock \doi{10.1145/2641190.2641198}.
\newblock URL \url{http://doi.acm.org/10.1145/2641190.2641198}.

\bibitem[Varshney \& Ramamurthy(2015)Varshney and
  Ramamurthy]{varshney2015persistent}
Varshney, Kush~R and Ramamurthy, Karthikeyan~Natesan.
\newblock Persistent topology of decision boundaries.
\newblock In \emph{Acoustics, Speech and Signal Processing (ICASSP), 2015 IEEE
  International Conference on}, pp.\  3931--3935. IEEE, 2015.

\bibitem[Wu et~al.(2016)Wu, Schuster, Chen, Le, Norouzi, Macherey, Krikun, Cao,
  Gao, Macherey, et~al.]{wu2016google}
Wu, Yonghui, Schuster, Mike, Chen, Zhifeng, Le, Quoc~V, Norouzi, Mohammad,
  Macherey, Wolfgang, Krikun, Maxim, Cao, Yuan, Gao, Qin, Macherey, Klaus,
  et~al.
\newblock Google's neural machine translation system: Bridging the gap between
  human and machine translation.
\newblock \emph{arXiv preprint arXiv:1609.08144}, 2016.

\bibitem[Zomorodian \& Carlsson(2005)Zomorodian and
  Carlsson]{zomorodian2005computing}
Zomorodian, Afra and Carlsson, Gunnar.
\newblock Computing persistent homology.
\newblock \emph{Discrete \& Computational Geometry}, 33\penalty0 (2):\penalty0
  249--274, 2005.

\bibitem[Zoph \& Le(2017)Zoph and Le]{45826}
Zoph, Barret and Le, Quoc~V.
\newblock Neural architecture search with reinforcement learning.
\newblock 2017.
\newblock URL \url{https://arxiv.org/abs/1611.01578}.

\end{thebibliography}

\newpage
\appendix

\section{Proofs, Conjectures, and Formal Definitions}

\subsection{Homology}

Homology is naturally described using the language of category theory. Let $Top^2$ denote the category of topological spaces and $Ab$ the category of abelian groups.

\begin{definition}[Homology Theory, \cite{bredon2013topology}]
	A homology theory on the on $Top^2$ is a function $H: Top^2 \to Ab$ assigning to each pair $(X,A)$ of spaces a graded (abelian) group $\{H_p(X,A)\}$, and to each map $f: (X, A) \to (Y,B)$, homomorphisms $f_*: H_p(X,A) \to H_p(Y,B)$, together with a natural transformation of functors $\partial_*: H_p(X,A) \to H_{p-1}(X,A)$, called the connecting homomorphism (where we use $H_*(A)$ to denote $H_*(A, \emptyset)$) such that the following five axioms are satisfied.
	\begin{enumerate}
		\item If $f \simeq g: (X,A) \to (Y,B)$ then $f_* = g_*: H_*(X,A) \to H_*(Y,B)$.
		\item For the inclusions $i: A \to X$ and $j: X \to (X,A)$ the sequence sequence of inclusions and connecting homomorphisms are exact.
		\item Given the pair $(X,A)$ and an open  set $U \subset X$ such that $cl(U) \subset int(A)$ then the inclusion $k:(X - U, A -U) \to (X,A)$ induces an isomorphism $k_*: H_*(X -U, A- U) \to H_*(X,A)$
		\item For a one point space $P, H_i(P) = 0$ for all $i \neq 0$.

		\item For a topological sum $X = +_\alpha X_\alpha$ the homomorphism
		\begin{equation*}
			\bigoplus(i_\alpha)_*: \bigoplus H_n(X_\alpha) \to H_n(X)
		\end{equation*}
		is an isomorphism, where $i_\alpha: X_\alpha \to X$ is the inclusion.
		
	\end{enumerate}
\end{definition}
For related definitions and requisite notions we refer the reader to \cite{bredon2013topology}.

\subsection{Proof of Theorem \ref{thm:hom_princ_gen}}
	\begin{theorem} Let $X$ be a topological space and $X^+$ be some open subspace. If $\scriptf \subset 2^X$ such that $f\in \scriptf$ implies $H_S(f) \neq H(X^+)$, then for all $f \in \scriptf$ there exists $A \subset X$ so that $f(A \cap X^+) = \{0\}$ and $f(A \cap (X \setminus X^+)) = \{1\}.$
	\end{theorem}
	\begin{proof}
		Suppose the for the sake of contraiction that for all $f\in \scriptf$, $H_S(f) \neq H(X^+)$ and yet there exists an $f$ such that for all $A \subset X$, there exists an $x\in A$ such that $f(x) = 1$. Then take $\scripta = \{x\}_{x \in X}$, and note that $f$ maps each singleton into its proper partition on $X$. We have that for any open subset of $V \subset X^+$, $f(V) = \{1\}$, and for any closed subset $W \subset X \setminus X^+$, $f(W) = \{0\}$. Therefore $X^+ = \bigcup_{A \in \tau_{X^+ \cap X}} A  \subset supp(f)$ as the subspace topology $\tau_{X^+ \cap X} = \tau_{X^+} \cap \tau_{X}$ where $\tau_{X^+} = \{A \in \tau_X\ |\ A \subset X^+\}$ and $\tau_X$ denotes the topology of $X$. Likewise, $int(X^-) \subset X \setminus supp(F)$ under the same logic.  Therefore $supp(f)$ has the exact same topology as $X^+$ and so by Theorem \ref{thm:top_to_hom} $H(X^+) = H(supp(f))$ but this is a contradiction.  This completes the proof.
	\end{proof}
\section{Additional Figures}
\begin{center}
	See next page.
\end{center}

\newpage
\begin{figure*}[!htb]
	\begin{center}
		\includegraphics[width=0.49\textwidth]{figures/convergence1.png}
		\includegraphics[width=0.49\textwidth]{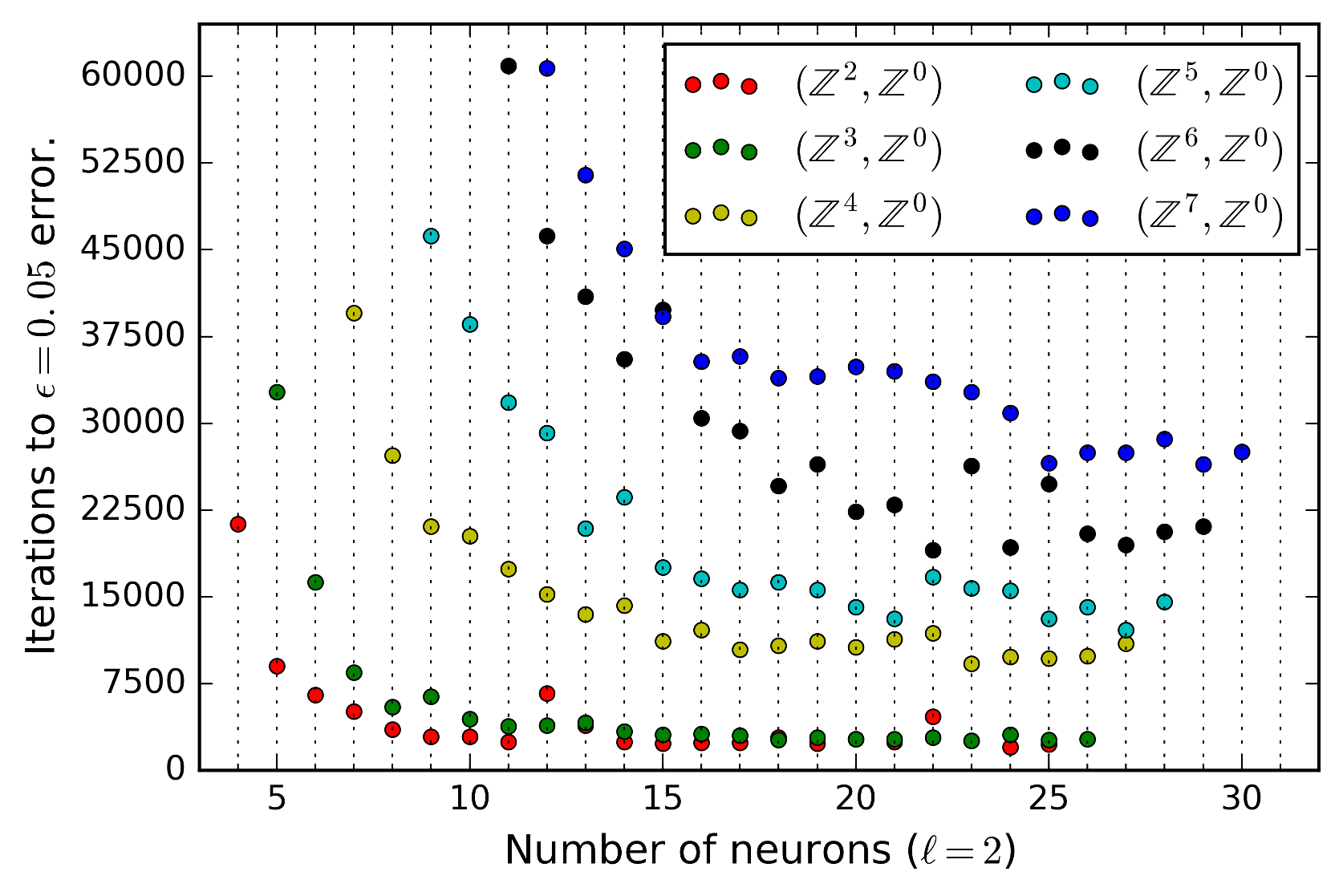}
		\includegraphics[width=0.49\textwidth]{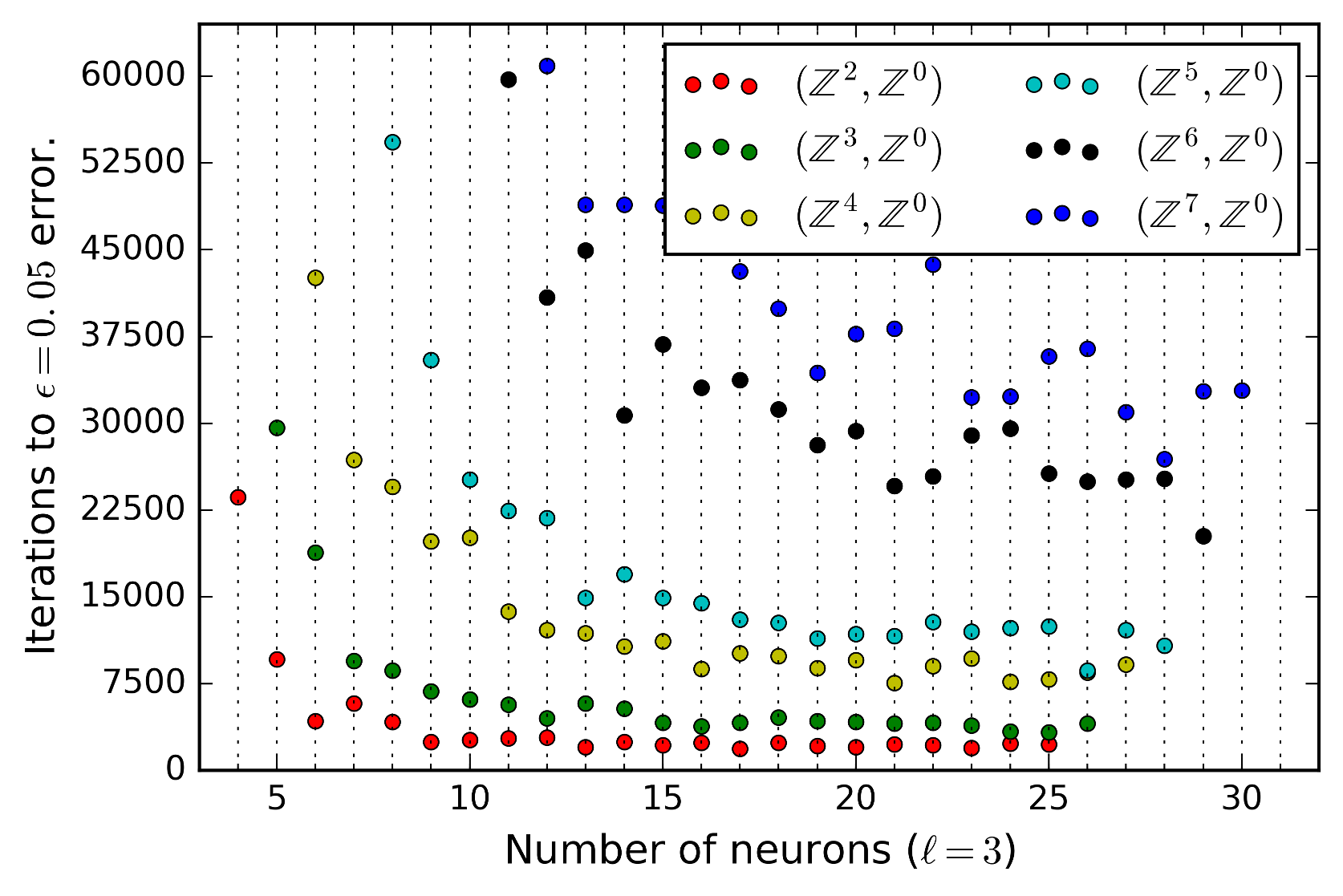}
		\includegraphics[width=0.49\textwidth]{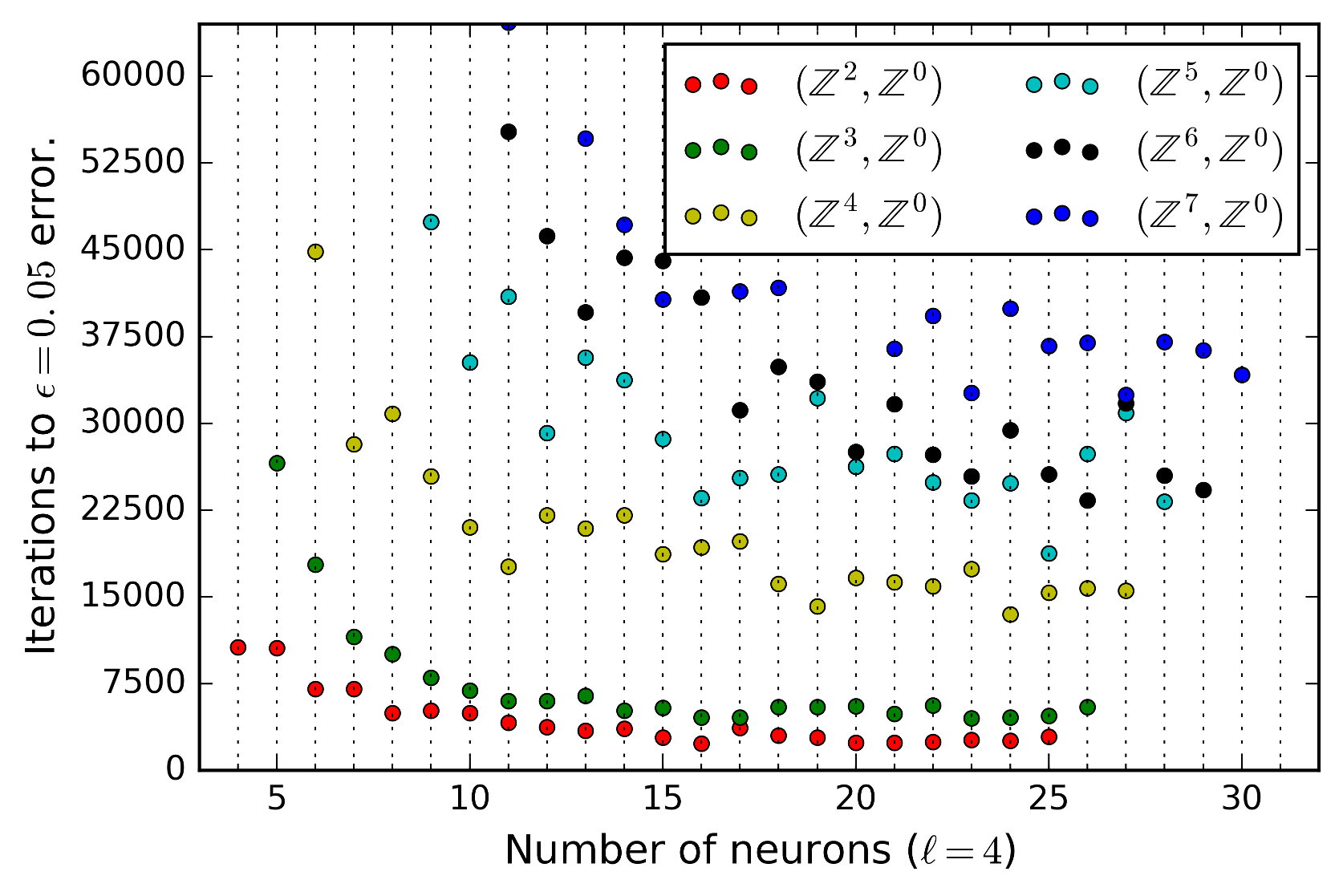}
	\end{center}
	\caption{A scatter plot of the number of iterations required for architectures of varying hidden dimension and number of layers to converge to 5\% misclassification error. The colors of each point denote the topological complexity of the data on which the networks were trained. Note the emergence of convergence bands.}
\end{figure*}

\begin{figure*}
	\begin{center}
		\includegraphics[width=0.99\textwidth]{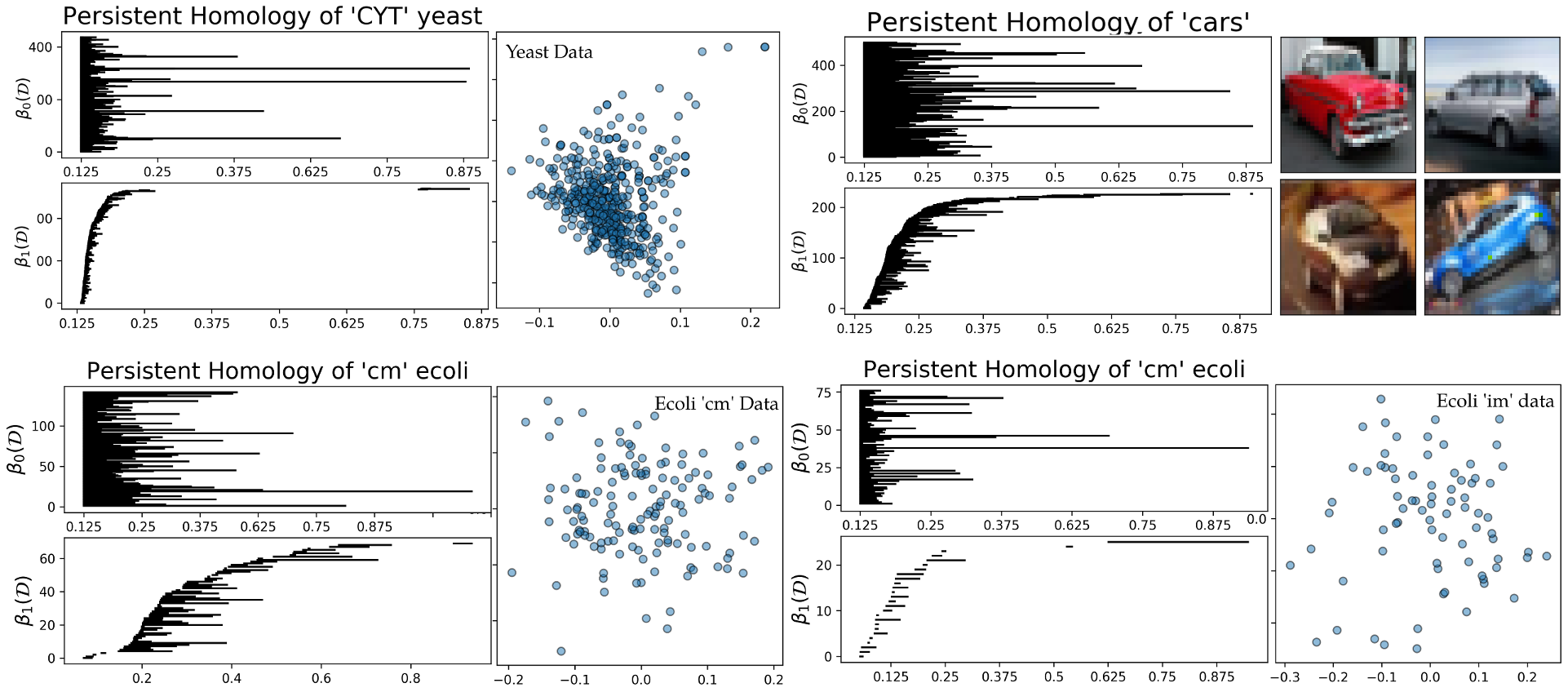}
		\caption{The topological persistence diagrams of several datasets. In each plot, the barcode diagram is given for $0$ and $1$ dimensional features along with a $2$-dimensional embedding of each dataset. Note that we do not compute the homologies of the datasets in the embeeding with the exception of CIFAR.}
	\end{center}
\end{figure*}

\begin{figure*}
	\begin{center}
	\includegraphics[height=0.40\textwidth]{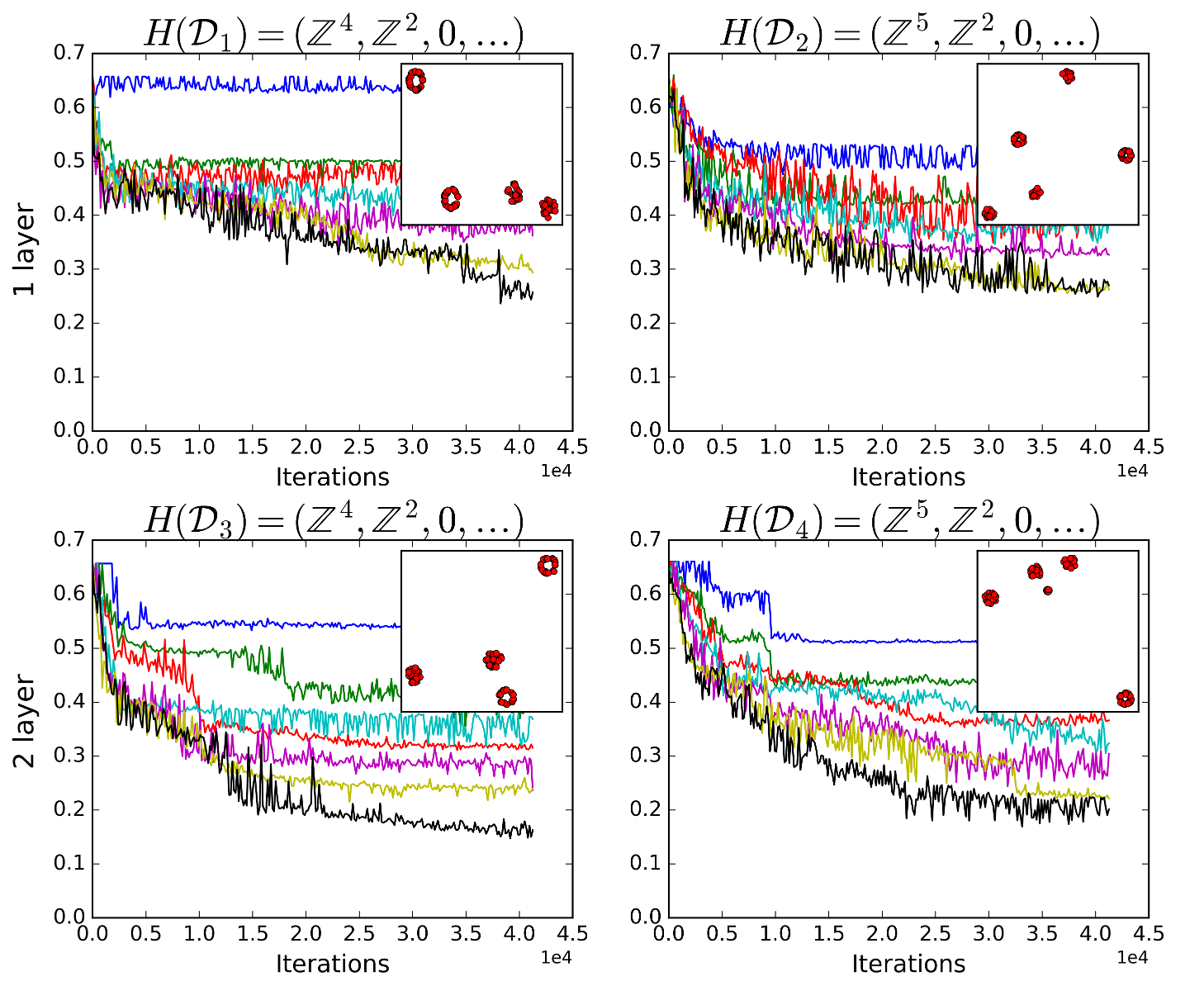}
	\includegraphics[height=0.40\textwidth]{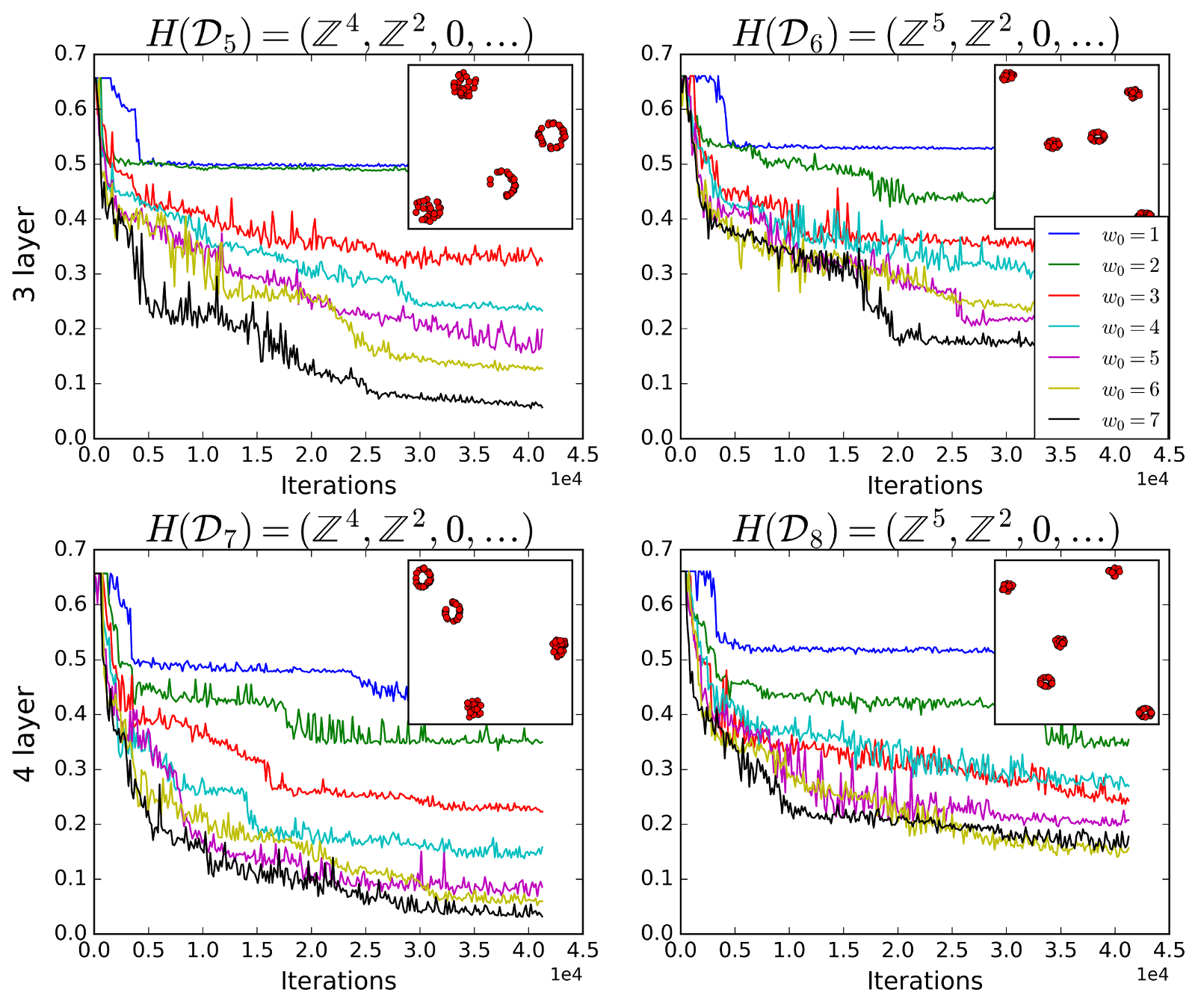}\\
	\includegraphics[height=0.40\textwidth]{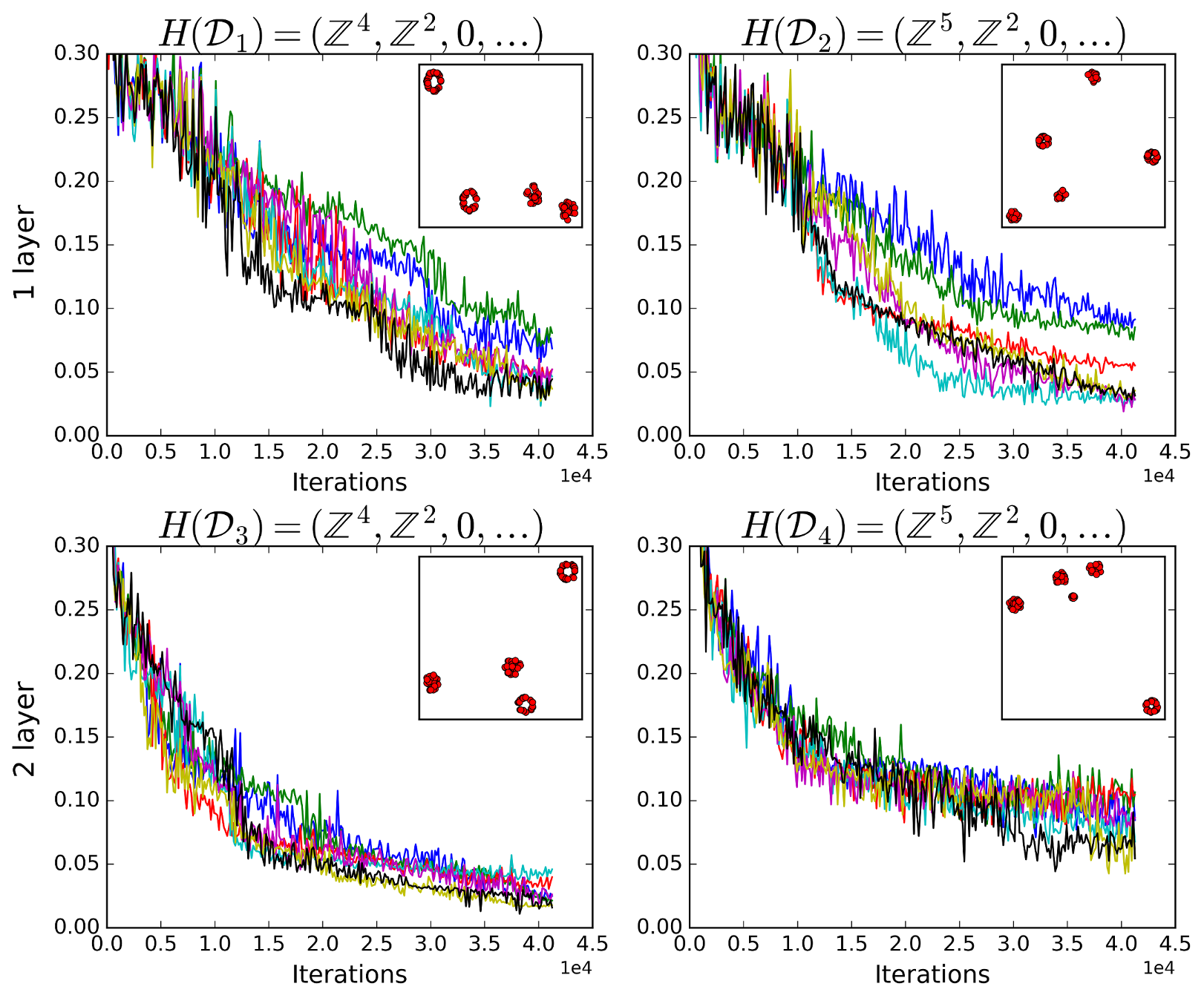}
	\includegraphics[height=0.40\textwidth]{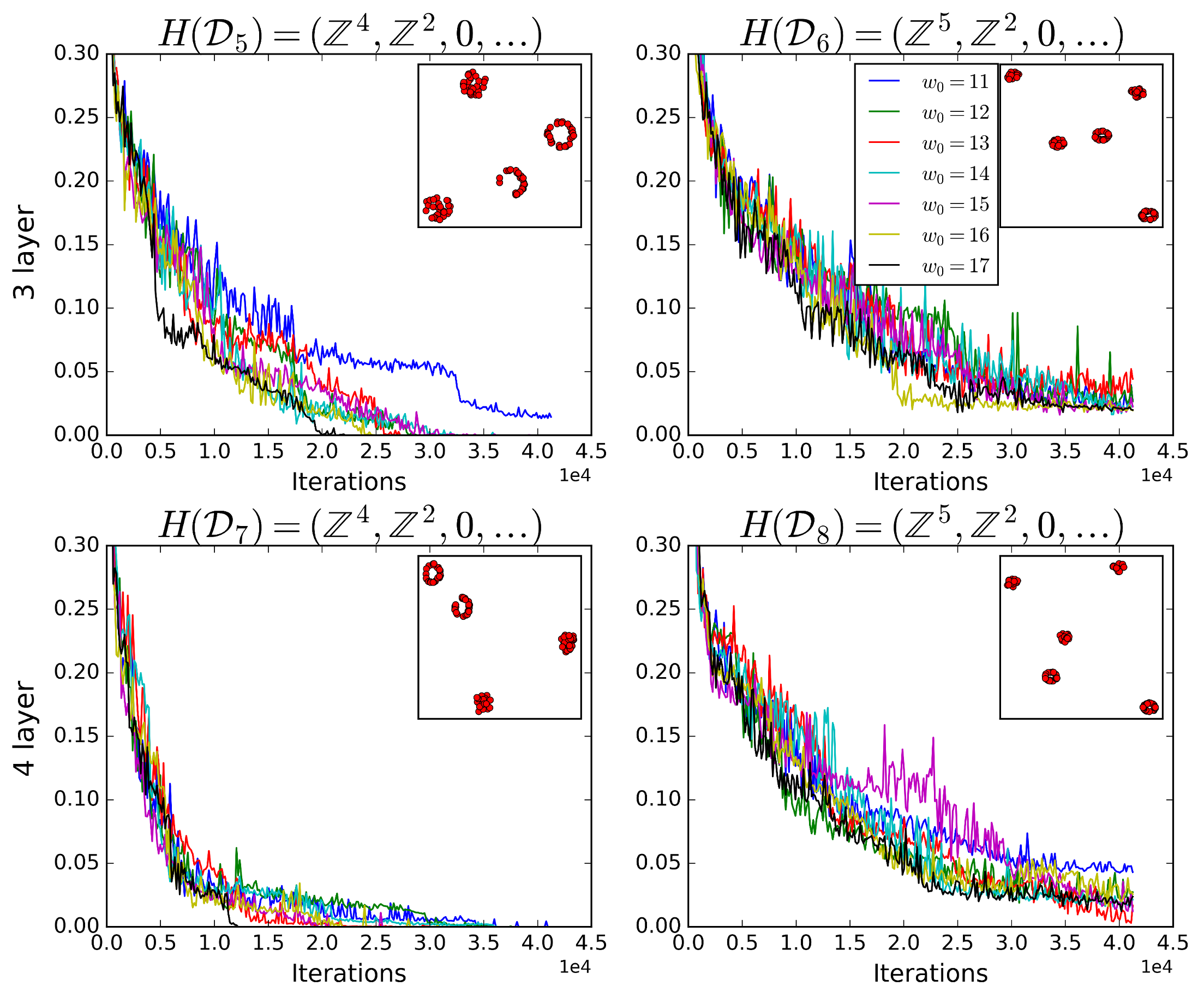}
	\end{center}

	\caption{Additional topological phase transitions in low dimensional neural networks as the homological complexity of the data increases. The upper right corner of each plot is a dataset on which the neural networks of increasing first layer hidden dimension are trained.  }
\end{figure*}

\end{document}